\documentclass[letterpaper]{article} % DO NOT CHANGE THIS
\usepackage{aaai2027}  % DO NOT CHANGE THIS
\usepackage[hyphens]{url}  % DO NOT CHANGE THIS
\usepackage{graphicx} % DO NOT CHANGE THIS
\usepackage{natbib}  % DO NOT CHANGE THIS AND DO NOT ADD ANY OPTIONS TO IT
\usepackage{caption} % DO NOT CHANGE THIS AND DO NOT ADD ANY OPTIONS TO IT
\usepackage{algorithm}
\usepackage{algorithmic}

\usepackage{newfloat}
\usepackage{listings}
\DeclareCaptionStyle{ruled}{labelfont=normalfont,labelsep=colon,strut=off} % DO NOT CHANGE THIS
\floatstyle{ruled}
\newfloat{listing}{tb}{lst}{}
\floatname{listing}{Listing}

\usepackage{booktabs}

\usepackage{longtable}
\usepackage{tabularx}
\usepackage{booktabs}
\usepackage{algorithm}
\usepackage{algorithmic}
\usepackage{amsmath}
\usepackage{amsfonts}
\usepackage{xcolor}
\usepackage[table]{xcolor}
\usepackage{amsthm}

\usepackage{mdframed}
\usepackage{fvextra}

\title{GraphReview: Scientific Paper Evaluation via LLM-based \\ Graph Evidence Expansion}
\author{
	Pujun Zheng\textsuperscript{\rm 1},
	Wanying Ren\textsuperscript{\rm 1},
	Jiacheng Yao\textsuperscript{\rm 1},
	Guoxiu He\textsuperscript{\rm 1} \thanks{Corresponding author.},
	Star X. Zhao\textsuperscript{\rm 2}
}
\affiliations{
	\textsuperscript{\rm 1}School of Economics and Management, East China Normal University\\
	\textsuperscript{\rm 2}Institute of Big Data, Fudan University\\
	\{pjzheng, wyren, jcyao\}@stu.ecnu.edu.cn, gxhe@fem.ecnu.edu.cn
}

\usepackage{bibentry}
\begin{document}
	
\nocopyright
\maketitle

\begin{abstract}

% **Scientific paper evaluation** requires assessing a manuscript in relation to both contemporaneous research and prior literature. However, existing LLM-based methods typically model these signals in isolation and lack a unified mechanism for propagating review evidence across papers. We propose \textbf{**GraphReview**}, a graph-baszed LLM framework that formulates paper evaluation as review-signal message passing over a semantic paper graph. The graph jointly captures **intrinsic quality, synchronic links** among contemporaneous papers, and **diachronic links** to prior work. LLMs estimate node-level quality priors and generate edge-level comparative evidence through pairwise paper comparisons, while Personalized PageRank integrates these signals for quality ranking, decision prediction, and review generation. To improve the quality of graph evidence, we introduce reward-induced maximum likelihood objectives for training the LLM backbones. Experiments show that GraphReview consistently outperforms the strongest baseline, with average improvements of 29.7\% on decision and ranking metrics, including gains of 23.7\% in Accuracy and 57.6\% in Spearman's $\rho$. It also generates higher-quality review texts and generalizes effectively across time periods and conference venues. The code is available at https://anonymous.4open.science/r/anonymous-GraphReview.

Scientific paper evaluation often involves not only assessing a manuscript itself, but also relating it to contemporaneous research and prior literature. However, existing LLM-based methods typically model these signals separately and lack a unified mechanism for aggregating review evidence across papers. We propose \textbf{GraphReview}, a graph-based LLM framework that formulates paper evaluation as inference-time graph evidence expansion over a semantic paper graph. The graph jointly captures intrinsic quality, synchronic links among contemporaneous papers, and diachronic links to prior work. LLMs are used to estimate node-level quality priors and generate edge-level comparative evidence through pairwise paper comparisons, while Personalized PageRank integrates these structured signals for quality ranking, decision prediction, and review generation. To produce higher-quality graph evidence, we propose reward-induced maximum likelihood objectives for training the LLM backbones. Experiments show that GraphReview consistently outperforms the strongest baseline, achieving average improvements of 29.7\% on decision and ranking metrics, including gains of 23.7\% in Accuracy and 57.6\% in Spearman's $\rho$. It also produces higher-quality review texts and generalizes effectively across time periods and conference venues. %The code is available at https://anonymous.4open.science/r/anonymous-GraphReview. % \href{https://anonymous.4open.science/r/anonymous-GraphReview}{anonymous GitHub}.
\end{abstract}


\section{Introduction}

% Pujun: 我希望强调，GraphReview 的架构核心是一种测试时扩展（inference-time expansion）机制，提供了一种结构化的方法用于拓展图信息。 LLM-based message passing mechanism 如何突破实现是 LLM 做 pairwise comparison + PPR 聚合的这一思考？
% Pujun：应明确任务是“给定整批投稿与已知容量的 ranking

% Guoxiu: Figure 1 没有在正文中refer
% Pujun: 突出 research gap 然后 ref
\begin{figure}[htbp]
	\setlength{\belowcaptionskip}{-5pt} 
	\setlength{\floatsep}{10pt} 
	\setlength{\textfloatsep}{10pt}
	\centering
	\includegraphics[width=\columnwidth]{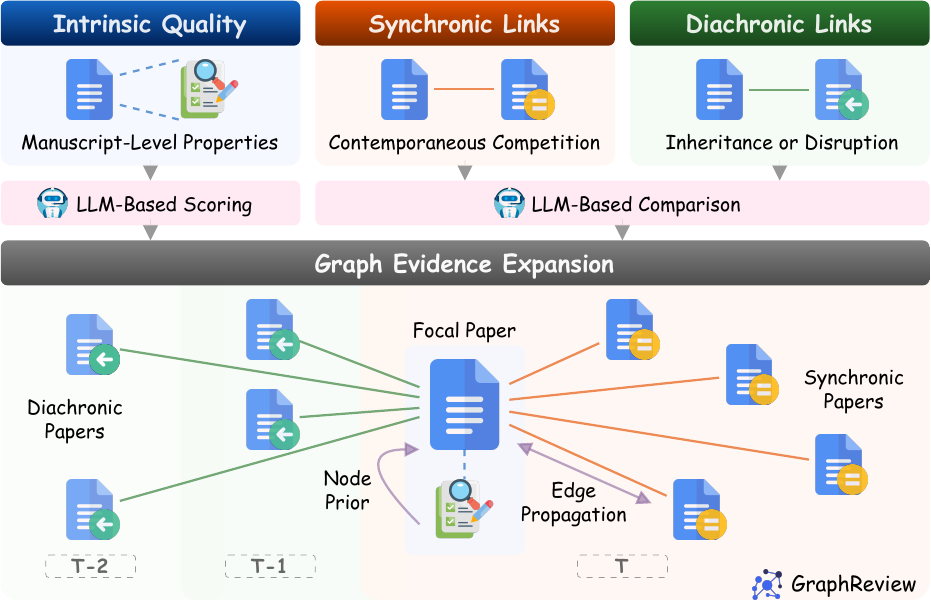}
	\caption{Previous LLM-based methods consider information sources in isolation (top), whereas GraphReview organizes multi-source review evidence in a semantic graph and aggregates it through inference-time expansion (bottom).}
	\label{fig:time_graph}
\end{figure}

Peer review plays a foundational role in ensuring publication quality \citep{alberts2008reviewing}. As the volume of scholarly literature grows rapidly \cite{he2023research}, leveraging artificial intelligence to assist paper reviewing has become an important strategy for addressing the reviewing crisis \citep{zhou2024llm, du2024llms, zhuang2025large, li2025can}. In particular, large language model (LLM)-based approaches have demonstrated remarkable capability and scalability in paper reviewing, and are increasingly adopted to support paper evaluations and help authors improve manuscript quality \cite{latona2024ai, thakkar2025can, thakkar2026large}.

Despite recent progress, existing LLM-based methods for paper evaluation remain limited by narrow modeling assumptions, as shown in Figure \ref{fig:time_graph} (top). The first line of work focuses on detailed analysis of textual content \cite{lu2024ai, jin2024agentreview, weng2025cycleresearcher, zhu2025deepreview, zeng2025reviewrl, chang2025treereview}, treating paper quality as an intrinsic property of the manuscript. The second line compares submissions within the same review cycle \cite{zhang2025from, zhao2025naipv2, zheng2026isolated}, defining quality in terms of contemporaneous competition. The third line evaluates new papers through their relationships to prior literature \cite{he2023h2cgl, xue2024predicting, zhao2025words}, thus grounding quality in a historical scholarly context while emphasizing inheritance, continuation, and disruption.

However, existing methods typically treat different sources of review evidence separately, without a unified mechanism for integrating signals across papers. From a science-of-science perspective, evaluating a paper requires not only assessing its writing quality and scientific merit, but also situating it within contemporaneous competition and the evolving landscape of domain knowledge \cite{fortunato2018science, wu2019large}. These signals are central to real-world peer review, especially in batch-style review scenarios where submissions are assessed comparatively, yet methods based on a single perspective capture only part of the picture.
% old: Existing LLM-based paper evaluation approaches, which mainly rely on end-to-end autoregressive generation or agent-based frameworks, further reinforce this limitation. Even when additional papers are incorporated, they are usually appended as unstructured context, rather than modeled as explicit comparative evidence that can be propagated across related papers. Consequently, these methods fail to capture and propagate comparative signals in the broader scientific context. This limitation motivates a graph-based framework that explicitly models context-aware, multi-source review signals and exploits their complementary strengths.
Existing LLM-based approaches to paper evaluation primarily rely on end-to-end autoregressive generation or agent-based frameworks, which further exacerbates this limitation. Even when additional papers are incorporated, they are typically appended as unstructured context rather than modeled as explicit comparative evidence that can be propagated across related work. As a result, these methods struggle to capture comparative signals from the broader scientific context. This limitation motivates a graph-based framework that uses the graph as a symbolic inference structure for inference-time expansion, enabling context-aware modeling of multi-source review signals and leveraging their complementarity.

% At the methodological level, GraphReview introduces graph learning into paper evaluation and leverages decoupled LLM backbones for message passing. Specifically, we construct a graph from semantic relations, use LLMs to assign each node a quality prior, and generate edge-level evidence through pairwise comparisons between connected nodes, with comparative signals propagated along edges. We then apply Personalized PageRank to combine global evidence while generating the review text in parallel, producing the final ranking, decision, and review.
% Pujun: 这段小改
% To address this, 
% paper quality
% To this end, we propose GraphReview, a graph-based LLM framework that jointly models multi-source review signals associated with paper quality. As illustrated in Figure \ref{fig:time_graph} (bottom), we represent paper evaluation within a graph structure and perform message passing over review evidence. % , thereby integrating three complementary sources of signals that have largely been studied in isolation
% old: To this end, we propose GraphReview, a graph-based LLM framework that represents paper evaluation as message passing over multi-source review evidence within a semantic paper graph.
To this end, we propose GraphReview, a graph-based LLM framework that represents paper evaluation as inference-time evidence expansion over multi-source review signals within a semantic paper graph.
% old: As illustrated in Figure \ref{fig:time_graph} (bottom), \textbf{Intrinsic Quality} captures manuscript-level properties such as originality, clarity, and significance, which provide the most direct signals for review; \textbf{Synchronic Links} model relations among contemporaneous submissions and capture comparative quality within the same review cycle, where strong papers are evaluated against competing work; \textbf{Diachronic Links} characterize intellectual linkages over time, capturing how a paper extends, diverges from, or overlaps with recent literature in ways that can either support or weaken its contributions. 
As illustrated in Figure \ref{fig:time_graph} (bottom), \textbf{Intrinsic Quality} captures manuscript-level properties such as originality, clarity, and significance; \textbf{Synchronic Links} model comparative relations among contemporaneous submissions; and \textbf{Diachronic Links} characterize how a paper extends, diverges from, or overlaps with recent literature.
 To operationalize this framework, GraphReview uses the graph to organize LLM-generated review signals rather than to learn a latent graph representation. The LLM first assigns each paper node a quality prior and then performs pairwise comparisons on selected edges to produce relational evidence. Personalized PageRank aggregates the pointwise priors and edge-level comparisons into a transductive global ranking and acceptance prediction under the target capacity, while review texts are generated in parallel from the same structured evidence. We further use reward-induced maximum likelihood as a stable soft-target objective to train the LLM signal generators.
 
% Jiacheng: intro部分背景部分我感觉写得很好，但是在介绍自己的方法的时候把top的三个部分又讲述了一遍（即Specifically, Intrinsic quality.....ways that can either support or weaken its contributions这部分），感觉这部分可以精简一点。
% Pujun: 很有道理但是，貌似有点难以调整的样子。不知道怎么办好。

% Guoxiu: 实证还是可以写的更细一些。毕竟这是Emnlp。
% 详细一些的实验描述，建议 *2 篇幅
% Overall, our method achieves leading performance with strong generalization across comparable datasets. It outperforms the strongest prior baselines by an average of 29.7\% across accuracy and ranking metrics, produces the highest-quality review text, and also demonstrates strong generalization across time periods and conference venues. 
We conduct comprehensive experiments from multiple perspectives to systematically evaluate the proposed framework. Results show that our method achieves leading overall performance, yielding an average relative gain of 29.7\% over the strongest baseline on decision and ranking metrics. Specifically, it improves Accuracy by 23.7\% and Spearman’s $\rho$ by 57.6\%. For review text quality, our method achieves win rates above 50\% against all baselines in technical depth, evidence grounding, scientific rigor, revision utility, and overall preference. Furthermore, cross-time-period and cross-venue generalization studies show that the model reliably distinguishes papers at different quality levels.
Our main contributions are as follows:

$\bullet$ We formulate paper evaluation in an integrated framework that captures intrinsic quality together with synchronic and diachronic relations.

% old: $\bullet$ We propose an LLM-based message passing mechanism tailored to paper evaluation, and design training objectives suitable for its LLM backbones.
$\bullet$ We propose an inference-time graph evidence expansion mechanism tailored to paper evaluation, and design training objectives suitable for its LLM backbones.

$\bullet$ Our method achieves leading performance on ranking and decision prediction, improves review text quality, and demonstrates robust generalization across time periods and conference venues.

\section{Related Work}

\subsection{Automated Scientific Paper Evaluation}

A major line of research evaluates the intrinsic quality of a paper by using LLM-based review agents to simulate peer review and produce automated pointwise scores \cite{jin2024agentreview,lu2024ai}. To improve review quality, many studies introduce iterative refinement or multi-round interaction \cite{weng2025cycleresearcher,tan2024peer}, domain-specific fine-tuning, reinforcement learning, and structured agent workflows \cite{zeng2025reviewrl,yu2024automated,tyser2024ai,garg2025revieweval,chang2025treereview,zhu2025deepreview,lu2025agent}.
Another line of work compares submissions within the same venue and adopts comparison-based evaluation. Representative directions include pairwise comparison for quality prediction \cite{zhao2025naipv2,hopner2025automatic} and preference aggregation via reinforcement learning with comparative rewards \cite{zhang2025from, zheng2026isolated, zheng2026navigating}. A further line of research considers scholarly inheritance, which is important for assessing a paper's actual contribution. This information has been explored in graph-based impact prediction \cite{he2023h2cgl,xue2024predicting}, but only a few LLM-based review systems incorporate paper retrieval or listwise ranking \cite{zhao2025words,zhu2025deepreview}.
Overall, most existing approaches rely on a single dominant perspective and do not combine intrinsic and relational signals within an integrated framework. Their autoregressive or agent-based architectures are also not well suited to explicitly organize multi-source relations, aggregate evaluative evidence, and support global decision making.

\subsection{LLM-Based Graph Methods}

Classic message-passing GNNs learn node representations and perform downstream predictions through neighborhood aggregation \cite{kipf2016semi,velivckovic2017graph,gilmer2017neural,hamilton2017inductive}. With the rapid development of LLMs, graph methods have become an increasingly important complementary component. Early work mainly converted graph structures into natural-language descriptions through hand-crafted rules \cite{wang2023can,zhao2023graphtext}. More recent studies pursue tighter alignment between graph structures and linguistic semantics, enabling LLMs to better understand graphs and generalize across reasoning tasks \cite{tang2024graphgpt,wang2025unigte,jing2026entropy,luo2025g,tao2025code,feng2025grapheval}. Another major direction uses graphs for retrieval-augmented generation, treating knowledge graphs as external memory and combining their structure with retrieval \cite{edge2024local,guo2024lightrag,gutierrez2024hipporag,dong2025youtu}. Other work models graphs as interactive environments for decision making \cite{zhang2023graph,finkelshtein2025actions}. However, existing LLM-based graph methods primarily use graph structures as context, external memory, or an interaction environment, rather than as a mechanism for organizing and aggregating quality evidence in evaluation tasks. As a result, they cannot explicitly represent high-level review signals, directional relationships, or evaluative evidence, making them less suitable for graph-based paper evaluation.

\section{Methodology}
\label{sec:methodology}

\subsection{Overview}
% ordered 和 ranking 放在一起也有点奇怪，应该是一个意思。
% Paper evaluation can be formulated as a sequence-to-ranking problem that maps a textual sequence to a global ranking. Specifically, given $n$ submitted papers for a conference and $(N-n)$ papers previously accepted at past conferences as an external knowledge base, the system takes $\mathbf{x}=[x_1,x_2,\cdots,x_n,x_{n+1},\cdots,x_N]$ as input and produces a quality-based ranking over the submitted papers, denoted by $\mathbf{r}=[r_1,r_2,\cdots,r_n]$.
% old: Scientific paper evaluation can be formulated as a sequence-to-ranking problem that maps a list of textual documents to a global quality ranking. Specifically, given $n$ concurrently submitted papers and an auxiliary corpus consisting of $(N-n)$ relevant historical papers, the system takes $\mathbf{x}=[x_1,x_2,\cdots,x_n,x_{n+1},\cdots,x_N]$ as input and produces a quality-based ranking of the submitted papers, denoted by $\mathbf{r}=[r_1,r_2,\cdots,r_n]$.
Scientific paper evaluation is formulated as a batch-style transductive ranking problem. Given the whole set of $n$ concurrently submitted papers, a known review capacity or acceptance rate $\gamma$, and an auxiliary corpus consisting of $(N-n)$ relevant historical papers, the system takes $\mathbf{x}=[x_1,x_2,\cdots,x_n,x_{n+1},\cdots,x_N]$ as input and produces a quality-based ranking of the submitted papers, denoted by $\mathbf{r}=[r_1,r_2,\cdots,r_n]$. The final decision selects the top $\lfloor\gamma n\rfloor$ papers from this batch, rather than evaluating each paper as an isolated standalone instance.
The overall process of our method is shown in Figure \ref{fig:process}.
% Guoxiu：这一段有点没必要。
% Pujun: 别人也有，还是保留了
%We first introduce our main LLM-based graph evidence expansion framework, then describe how we construct the graph and perform computation over it. We finally presents the training strategy for the LLM backbones. 

% Guoxiu：这一部分好大啊，感觉字号比正文都大了，可以把字体改小一点。
% \newcommand{\midsmall}{\fontsize{10pt}{12pt}\selectfont}
\begin{algorithm}[t]
    \small
	\caption{Inference-Time Graph Evidence Expansion} 
	\label{alg:message_passing}
	\begin{algorithmic}[1]
		\REQUIRE Minimum improvement threshold $\epsilon$ and maximum patience $c_{\max}$
		\ENSURE Best round $T^*$ and best ranking $\mathbf r^*$
		\STATE \textbf{Initialize} $\eta^* \gets 0$, $T^* \gets 1$, $\mathbf r^* \gets \emptyset$, $T \gets 1$, $c \gets 0$
		\WHILE{$c < c_{\max}$}
		\STATE Construct adjacency matrix $\mathbf{C}^{(T)}$ by $f_{\mathrm{S2FM}}$
		\FOR{each node $v$}
		\FOR{each $u$ satisfying $c_{uv}^{(T)}=1$}
		\STATE {\color{gray}\texttt{// Preference generation}}
		\STATE $g_{v \leftarrow u} = f_{\mathrm{LLM}}(x_u,x_v,p_{\mathrm c})$
		\ENDFOR
		\STATE {\color{gray}\texttt{// Evidence collection}}
		\STATE $\mathbf{g}_v=\bigoplus\limits_{c_{uv}^{(T)}=1} g_{v \leftarrow u}$
		\ENDFOR
		\STATE {\color{gray}\texttt{// Ranking aggregation}}
		\STATE $\boldsymbol{\pi}=f_{\mathrm{PPR}}\Big(\big\{f_{\mathrm{LLM}}(x_v,p_{\mathrm{s}})\big\},\big\{\mathbf{g}_v\big\}\Big)$
		\STATE $\mathbf r = \operatorname{argsort}(\boldsymbol{\pi})$
		\STATE Calculate the performance $\eta$ from $\mathbf{r}$
		\IF{$\eta - \eta^* > \epsilon$}
		\STATE $\eta^* \gets \eta$, \ $T^* \gets T$, \ $\mathbf r^* \gets \mathbf r, c \gets 0$
		\ELSE
		\STATE $\ c \gets c+1$
		\ENDIF
		\STATE $T \gets T + 1$
		\ENDWHILE
		\RETURN $T^*, \mathbf r^*$
	\end{algorithmic}
\end{algorithm}

\subsection{Graph Evidence Expansion}
\label{sec:message_passing}

From a graph perspective, a Transformer input sequence with global attention can be viewed as a fully connected graph \cite{joshi2025transformers}, and attention can be interpreted as a form of GNN computation \cite{frasca2025neural}. Although inspired by this analogy, our method differs from conventional GNNs: the graph acts as a symbolic inference structure rather than a learned latent-representation model. Specifically, we formulate paper evaluation as graph-structured evidence expansion, where task-specific pairwise judgments are produced directly by LLMs. As summarized in Algorithm~\ref{alg:message_passing}, inference consists of three message-passing-style stages: preference generation, evidence collection, and ranking aggregation.

We perform inference over a dynamically expanded graph that combines intrinsic quality, synchronic links, and diachronic links. Across rounds, S2FM activates semantic edges, the LLM generates edge-level pairwise evidence, and PPR aggregates node priors with relational evidence. Increasing $T$ thus expands structured inference-time evidence, rather than adding recurrent GNN-style state updates, leading to improved and eventually convergent performance.

\paragraph{Preference Generation}

The edge-evidence function characterizes how a neighboring paper provides comparative evidence for a target paper. Unlike conventional GNNs that propagate learned embeddings or recurrent node states, our method uses the LLM to perform pairwise comparisons over selected edges. Given the textual content $x_u$ and $x_v$ of nodes $u$ and $v$, we query the LLM $f_{\mathrm{LLM}}$ with a comparison prompt $p_{\mathrm c}$ to infer their relative quality. Each comparison returns two parts: a preference label, and a natural language rationale retained in $g_v$ for text consolidation. Graph connectivity is specified by a dynamically expanded adjacency matrix $\mathbf{C}^{(T)}$. Each entry $c^{(T)}_{uv}$ indicates whether the semantic link between nodes $u$ and $v$ is activated by the Sequential 2-Factor Matching (S2FM) algorithm, denoted by $f_{\mathrm{S2FM}}$. Since evidence is generated for both comparison directions, the adjacency matrix is symmetric, that is, $c^{(T)}_{uv} = c^{(T)}_{vu}$.

\paragraph{Evidence Collection}

The collection function gathers all edge-level evidence associated with a target node. Instead of compressing incoming signals with a pooling operator, we retain the full set of neighborhood evidence through concatenation. Let $\mathcal{N}(v)$ denote the neighbor set of node $v$. The aggregated evidence for node $v$ is formed by concatenating the comparison results from all neighbors in $\mathcal{N}(v)$.

\paragraph{Ranking Aggregation}

The ranking aggregation step combines pointwise and relational evidence to produce the final ranking signal. Rather than training an end-to-end graph predictor, we integrate LLM-based pointwise priors with relational evidence through Personalized PageRank (PPR)  \cite{jeh2003scaling}, denoted by $f_{\mathrm{PPR}}$, which yields a global ranking without training an entire graph model. For each node $v$, we query the LLM $f_{\mathrm{LLM}}$ with the prompt $p_{\mathrm s}$ to obtain a direct quality estimate for $x_v$. This score serves as a prior importance signal that captures standalone evidence for the paper. We then refine it with comparative evidence collected from neighboring nodes. Specifically, PPR diffuses node importance over the graph based on both the pointwise priors and the aggregated edge evidence. Finally, we apply $\mathrm{argsort}$ to the converged PPR scores $\boldsymbol{\pi}$ to obtain a total ranking, and predict the top $\lfloor \gamma n \rfloor$ papers as accepted under a preset acceptance rate $\gamma$, with the remaining papers rejected.

% Guoxiu：为啥没有正文ref这个figure？另外，没有一个图来整体的串联起3.2、3.3、3.4
% Pujun: 这个图放上去吃掉我7-8行
\begin{figure*}[htbp]
	\setlength{\belowcaptionskip}{-10pt} 
	\setlength{\floatsep}{10pt} 
	\setlength{\textfloatsep}{10pt}
	\centering
	\includegraphics[width=\textwidth]{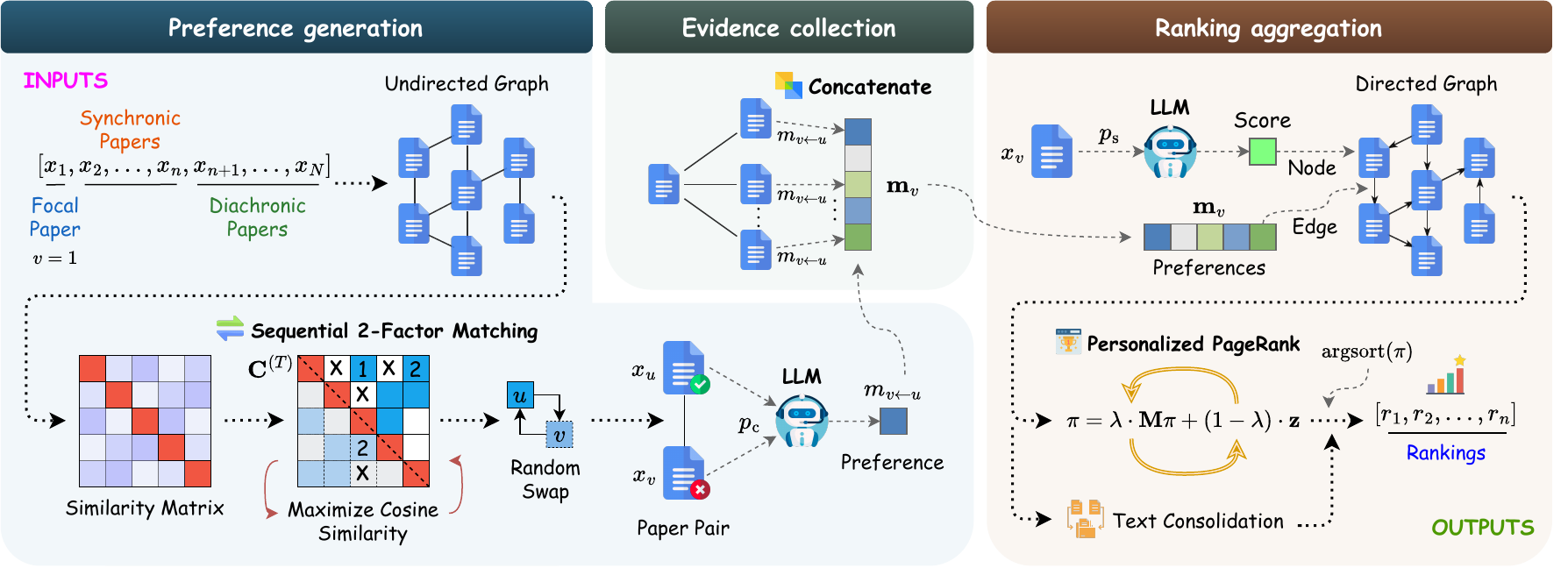}
	\caption{The overall process of our method. It includes preference generation (left), evidence collection (center), and ranking aggregation (right). The input is a list of papers, and the output is the ranking of the papers.}
	\label{fig:process}
\end{figure*}

\subsection{Graph Building and Computing}
\label{sec:graph_building_and_computing}

This section describes how we design $f_{\mathrm{S2FM}}$ to construct graphs and use $f_{\mathrm{PPR}}$ to perform computation on them.

\paragraph{Sequential 2-Factor Matching}

% 描述 S2FM 的情况下，添加三个优秀性质（直接写出名字即可，不用说明），参考自rebuttal的内容。

S2FM progressively expands the graph through repeated 2-factor matching steps, maintaining three desirable properties: scaling consistency, permutation equivariance, and guaranteed connectivity. Let $\mathbf{C}^{(t)}\in\mathbb{R}^{N\times N}$ denote the matching matrix at iteration $t$, where $c_{uv}^{(t)}$ indicates whether an edge exists between nodes $u$ and $v$. Each iteration adds $N$ bidirectional edges. At iteration $t$, S2FM solves a maximum-weight 2-factor matching problem, selecting a binary symmetric increment matrix $\Delta\mathbf{C}^{(t)}$ that maximizes the total semantic similarity $\sum_{u,v}\mathbf{h}_u^\top\mathbf{h}_v\Delta c_{uv}^{(t)}$ while enforcing that each node has degree 2. To avoid repeated edges across iterations, the constraint $\Delta c_{uv}^{(t)} \le 1-c_{uv}^{(t-1)}$ ensures that newly selected edges connect only previously unmatched node pairs, after which the adjacency matrix is updated as $\mathbf{C}^{(t)}=\mathbf{C}^{(t-1)}+\Delta\mathbf{C}^{(t)}$. After $T$ iterations, S2FM produces the final adjacency matrix $\mathbf{C}^{(T)}$.

\paragraph{Personalized PageRank}

We introduce PPR to unify node-level pointwise predictions and edge-level pairwise preference constraints produced by the aggregation function. PPR performs a weighted random walk between the pointwise prior and local transition probabilities, yielding a global ranking score. For node-level evaluation, we first construct a normalized node prior $\mathbf{z}=[z_1,z_2,\cdots,z_N]$ from pointwise prediction scores $\mathbf{e}=[e_1,e_2,\cdots,e_N]$. We use a small constant $\epsilon$ to handle potentially non-positive scores or missing values, thereby ensuring a valid probability distribution:
\begin{equation}
	z_u=\frac{\max(e_u,\epsilon)}{\sum_{k=1}^N \max(e_k,\epsilon)}
\end{equation}

For edge-level comparisons, we initialize a graph with the bidirectional adjacency matrix $\mathbf{C} \gets \mathbf{C}^{(T)}$. For each edge $(x_u, x_v)$, if $f_{\mathrm{LLM}}$ judges that $x_u$ strictly outperforms $x_v$, we remove the directed edge from $u$ to $v$ by setting:
\begin{equation}
	c_{uv}\gets 0,\quad \text{if } x_u\succ x_v
\end{equation}

The resulting adjacency matrix $\mathbf{C}$ represents a directed graph. Let $\mathbf{D}\in\mathbb{R}^{N\times N}$ be the diagonal out-degree matrix with $d_{vv}=\sum_{k=1}^N c_{vk}$. To construct a valid Markov chain, we define a column-normalized transition matrix $\mathbf{M}\in\mathbb{R}^{N\times N}$. For nodes with zero out-degree, we redirect according to the prior $\mathbf{z}$. Define each element $m_{uv}$ as:
\begin{equation}
	{m}_{uv}=
	\begin{cases}
		\frac{c_{vu}}{d_{vv}} & \text{if } d_{vv}>0 \\
		z_u & \text{if } d_{vv}=0
	\end{cases}
\end{equation}

The final aggregated score vector $\boldsymbol{\pi}\in\mathbb{R}^N$ is the stationary distribution of this random walk, combining local pairwise transitions $\mathbf{M}$ with the global pointwise prior $\mathbf{z}$ via:
\begin{equation}
	\boldsymbol{\pi}=\lambda\cdot \mathbf{M}\boldsymbol{\pi}+(1-\lambda)\cdot \mathbf{z}
\end{equation}

Here, $\lambda\in(0,1)$ is the damping factor. We approximate $\boldsymbol{\pi}$ via power iteration.

In addition, since PPR cannot directly propagate or update textual rationales, we run a parallel text consolidation pipeline that merges node and edge explanations alongside PPR’s aggregation of numerical signals.

\subsection{LLM Backbones}
\label{sec:llm_backbones} 

This section describes the training method for the LLM backbones used to generate node-level priors and edge-level comparative evidence, denoted by $f_\mathrm{LLM}$.

\paragraph{Cold Start and Prompt Optimization}
To encourage the model to learn not only superficial response patterns but also the underlying evaluation criteria, we first perform cold-start supervised fine-tuning. Paper evaluation, however, is a specialized and cognitively demanding task, and even human experts may struggle to articulate criteria that are both effective and comprehensive. To construct prompts that cover the major dimensions of paper assessment, we adopt a simple self-optimization procedure that iteratively refines the prompt, and then use the optimized prompt to generate training responses. Training on these cold-start data enables the model not only to produce a special review-signal token representing the evaluation outcome, but also to explain the evaluation dimensions and key considerations that lead to the same conclusion.

\paragraph{Reward-Induced Maximum Likelihood}
% Pujun: Figure 5 要引用

% old: Many existing approaches apply reinforcement learning with verifiable rewards and rollout sampling after the cold-start stage. However, these methods are not well suited for paper evaluation. Unlike tasks such as logical or mathematical reasoning, ground truth in paper evaluation comes from expert scores, where judgments are driven more by holistic preference than by long-horizon sequential decision making. As a result, these methods are computationally expensive and often produce sparse and noisy preference signals, resulting in unstable optimization. To address this, we propose Reward-Induced Maximum Likelihood (RIML), which is inspired by reinforcement learning methods that define rewards based on the discrepancy between the current prediction and the ground truth. Rather than optimizing these rewards directly, RIML converts them into target distributions over the review-signal token classes and trains the cold-started model with a maximum likelihood objective. Specifically, we define the objectives as follows.
After the cold-start stage, we use Reward-Induced Maximum Likelihood (RIML) as a simple and stable soft-target objective for aligning the LLM signal generators. Unlike rollout-based reinforcement learning, RIML does not require a costly reward pipeline or long-horizon exploration. It converts expert score rewards and pairwise preferences into target distributions over review-signal token classes, and then trains the cold-started model with maximum likelihood. This formulation better matches paper evaluation, where supervision comes from human rating distributions and comparative judgments rather than verifiable multi-step rewards. Specifically, the objectives are defined as follows.

% Pujun：把这两条公式更换成后面的 Objectives，在写法上作做重整

% old: At the node level, the model predicts a distribution over a set of discrete score anchors that approximate the continuous score space. Supervision is given by a reward-induced soft target distribution over anchors, which preserves score proximity and provides smoother learning signals than hard assignment. The node-level objective is:
% old: \begin{equation}
% old: 	\mathcal{L}_{\mathrm{node}}(\theta)
% old: 	=
% old: 	-\frac{1}{N}\sum_{v=1}^{N}\sum_{k=1}^{K} w_v\, y_{vk}\log P(k\mid q_v)
% old: \end{equation}
% old: Here, $y_{vk}$ denotes the reward-induced target probability for anchor $a_k$ for sample $v$, and $P(k\mid q_v)$ is the model prediction over classes. We fix the sample weight $w_v=1$ for all node samples, and $N$ is the total number of node samples.
At the node level, the model predicts a distribution over discrete score anchors $\{a_k\}_{k=1}^{K}$. For a paper $v$ with score $s_v$, we define:
\begin{equation}
\begin{aligned}
	P(k\mid q_v)&=\frac{\exp(z_{vk})}{\sum_{l=1}^{K}\exp(z_{vl})},\\
	y_{vk}&=\frac{\exp(r_{vk})}{\sum_{l=1}^{K}\exp(r_{vl})},\quad r_{vk}=-\frac{(s_v-a_k)^2}{2\sigma^2} \\
	\mathcal{L}_{\mathrm{node}}(\theta)
	&=
	-\frac{1}{N}\sum_{v=1}^{N}\sum_{k=1}^{K} w_v y_{vk}\log P(k\mid q_v)
\end{aligned}
\end{equation}
Here, the distance-based reward induces a soft target over score anchors, preserving proximity in the continuous score space; we set $w_v=1$ for all node samples.

% old: At the edge level, we model relative quality by constructing binary comparisons between samples with sufficiently different ground-truth scores. Each valid pair is assigned a one-hot label according to the score order, and is weighted by the score gap to emphasize more informative comparisons. The edge-level objective is:
% old: \begin{equation}
% old: 	\mathcal{L}_{\mathrm{edge}}(\theta)
% old: 	=
% old: 	-\frac{1}{M}\sum_{m=1}^{M}\sum_{k=0}^{1} w_{uv}\, y_{uvk}\log P(k\mid q_{uv})
% old: \end{equation}
% old: Here, $y_{uvk}$ refers to the one-hot target distribution induced by the pairwise comparison label, and $P(k\mid q_{uv})$ is the prediction over classes. $w_{uv}=|s_u-s_v|$ is the weight determined by the ground-truth score gap, and each pair $(u,v)$ corresponds to an edge index $m$.
At the edge level, for a pair $(u,v)$ with sufficiently different ground-truth scores, we define:
\begin{equation}
\begin{aligned}
	P(k\mid q_{uv})&=\frac{\exp(z_{uvk})}{\exp(z_{uv0})+\exp(z_{uv1})},\\
	y_{uvk}&=r_{uvk}, \quad r_{uvk}=\mathbb{I}\big[k=\mathbb{I}[s_u-s_v>0]\big],\\
	\mathcal{L}_{\mathrm{edge}}(\theta)
	&=
	-\frac{1}{M}\sum_{m=1}^{M}\sum_{k=0}^{1} w_{uv}y_{uvk}\log P(k\mid q_{uv})
\end{aligned}
\end{equation}
Here, $w_{uv}=|s_u-s_v|$ weights clearer pairwise preferences more strongly, and each valid pair corresponds to one edge index $m$.

\section{Experiments}
\label{sec:exp}

\begin{table*}[t]
	\renewcommand{\arraystretch}{0.82}
	\setlength{\belowcaptionskip}{-8pt}
	\setlength{\floatsep}{10pt}
	\setlength{\textfloatsep}{10pt}
	\setlength{\aboverulesep}{1.5pt}
	\setlength{\belowrulesep}{1.5pt}
	\centering
	\scriptsize
	\setlength{\tabcolsep}{3pt}
	\begin{tabularx}{\textwidth}{@{}
			>{\raggedright\arraybackslash}p{0.16\linewidth}
			*{5}{>{\centering\arraybackslash}X}
			*{7}{>{\centering\arraybackslash}X}@{}}
		
		\toprule
		
		& \multicolumn{5}{c}{\textbf{Method Characteristics}}
		& \multicolumn{3}{c}{\textbf{Decision Performance}}
		& \multicolumn{3}{c}{\textbf{Ranking Performance}}
		& \\
		\cmidrule(lr){2-6}
		\cmidrule(lr){7-9}
		\cmidrule(lr){10-12}
		\textbf{Method}
		& \textbf{Intrinsic Quality}
		& \textbf{Syn-chronic Links}
		& \textbf{Dia-chronic Links}
		& \textbf{Review Text}
		& \textbf{Data Leakage Risk}
		& \textbf{Accuracy}
		& \textbf{F1}
		& \textbf{AUC}
		& \textbf{Spearman $\rho$}
		& \textbf{Kendall $\tau$}
		& \textbf{NDCG @10}
		& \textbf{Avg. Performance} \\
		\specialrule{\lightrulewidth}{0.5pt}{0.5pt}
		
		\rowcolor{gray!10}
		\multicolumn{13}{@{}l}{\textbf{Classic GNNs}} \\
		GCN & $\checkmark$ & $\checkmark$ & $\checkmark$ &  & $\circ$
		& 0.5980 & 0.5404 & 0.5720 & 0.1267 & 0.0852 & 0.5552 & 0.4129 \\
		GAT & $\checkmark$ & $\checkmark$ & $\checkmark$ &  & $\circ$
		& 0.5760 & 0.5235 & 0.5104 & 0.0485 & 0.0349 & 0.5734 & 0.3778 \\
		GraphSAGE & $\checkmark$ & $\checkmark$ & $\checkmark$ &  & $\circ$
		& 0.5640 & 0.5143 & 0.5055 & 0.0492 & 0.0348 & 0.5868 & 0.3758 \\
		\specialrule{\lightrulewidth}{0.5pt}{0.5pt}
		
		\rowcolor{gray!10}
		\multicolumn{13}{@{}l}{\textbf{ LLMs}} \\
		GPT-5-Mini & $\checkmark$ &  &  & $\checkmark$ & $\bullet$
		& 0.7100 & 0.6605 & 0.7104 & 0.4136 & \underline{0.3317} & 0.6982 & 0.5874 \\
		Gemini-2.5-Flash & $\checkmark$ &  &  & $\checkmark$ & $\bullet$
		& 0.6060 & 0.5387 & 0.5557 & 0.1936 & 0.1569 & 0.6422 & 0.4488 \\
		Deepseek-V3.2 & $\checkmark$ &  &  & $\checkmark$ & $\bullet$
		& 0.6180 & 0.5527 & 0.5975 & 0.3536 & 0.2908 & 0.6614 & 0.5123 \\
		\specialrule{\lightrulewidth}{0.5pt}{0.5pt}
		
		\rowcolor{gray!10}
		\multicolumn{13}{@{}l}{\textbf{Review Agents}} \\
		AgentReview & $\checkmark$ &  &  & $\checkmark$ & $\circ$
		& 0.5160 & 0.5074 & 0.5528 & 0.0042 & 0.0026 & 0.6411 & 0.3707 \\
		AIScientist & $\checkmark$ &  &  & $\checkmark$ & $\circ$
		& 0.7100 & 0.5417 & 0.6418 & 0.3162 & 0.2508 & 0.7493 & 0.5350 \\
		SEA-7B & $\checkmark$ &  &  & $\checkmark$ & $\circ$
		& 0.6960 & 0.4104 & 0.5459 & 0.0983 & 0.0754 & 0.6585 & 0.4141 \\
		CycleReviewer-7B & $\checkmark$ &  &  & $\checkmark$ & $\circ$
		& 0.6620 & 0.5300 & 0.6766 & 0.3132 & 0.2255 & 0.6859 & 0.5155 \\
		DeepReview-7B & $\checkmark$ &  & $\triangle$ & $\checkmark$ & $\circ$
		& 0.6540 & 0.5470 & 0.5890 & 0.2928 & 0.2110 & 0.6938 & 0.4979 \\
		DeepReview-14B & $\checkmark$ &  & $\triangle$ & $\checkmark$ & $\circ$
		& 0.6860 & 0.6240 & 0.6494 & 0.3995 & 0.2991 & 0.6657 & 0.5540 \\
		\specialrule{\lightrulewidth}{0.5pt}{0.5pt}
		
		\rowcolor{gray!10}
		\multicolumn{13}{@{}l}{\textbf{Comparative Systems}} \\
		PairReview &  & $\checkmark$ &  & $\checkmark$ & $\circ$
		& 0.6700 & 0.5701 & 0.6047 & 0.2585 & 0.1781 & 0.6644 & 0.4910 \\
		CNPE-7B &  & $\checkmark$ &  & $\checkmark$ & $\circ$
		& 0.7200 & 0.6692 & 0.7363 & 0.3995 & 0.2774 & \underline{0.8040} & 0.6010 \\
		NAIP-8B &  &  & $\checkmark$ &  & $\circ$
		& 0.6140 & 0.5413 & 0.5641 & 0.1585 & 0.1074 & 0.6882 & 0.4456 \\
		NAIPv2-8B &  & $\checkmark$ &  &  & $\circ$
		& \underline{0.7260} & \underline{0.6780} & \underline{0.7627}
		& \underline{0.4205} & 0.2928 & 0.7723 & \underline{0.6087} \\
		\specialrule{\lightrulewidth}{0.5pt}{0.5pt}
		
		\rowcolor{gray!10}
		\multicolumn{13}{@{}l}{\textbf{Ours}} \\
		GraphReview & $\checkmark$ & $\checkmark$ & $\checkmark$ & $\checkmark$ & $\circ$
		& \textbf{0.8980} & \textbf{0.8806} & \textbf{0.9590}
		& \textbf{0.6626} & \textbf{0.4808} & \textbf{0.8547}
		& \textbf{0.7893} \\
		\bottomrule
	\end{tabularx}
	\caption{Comparison of method characteristics and performance.
		$\checkmark$ indicates the use of a corresponding information source or
		the availability of a functionality; $\triangle$ denotes a feature that has
		been developed but is not enabled by default in the released evaluation
		setting; and $\bullet$/$\circ$ indicate high/low risks of data leakage. % $^{*}$ indicates a statistically significant difference from the second-best result (Bootstrap $p < 0.001$). 
		\textbf{Best} and \underline{second-best} results are highlighted,
		respectively.}
	\label{tab:main}
\end{table*}

\subsection{Experimental Setup}

% Guoxiu：有多少数据还是要在正文提一下。
\paragraph{Dataset Construction}
All data are collected via the official OpenReview API at \url{https://openreview.net}. We construct the primary training and test sets from ICLR 2025 submissions, using the average human reviewer score as the ground-truth label. The external auxiliary knowledge base consists of all accepted papers from ICLR, ICML, and NeurIPS in 2023 and 2024, covering recent advances and the research frontier of the field. For each batch, we create a graph with $N$ 1.5k nodes for $n$ 0.5k submissions. 
This takes about 4 hours on two RTX Pro 6000 GPUs.
To assess generalization across time periods and venues, we further sample papers from ICLR 2026 and ICML 2025. These data include expert-provided paper group annotations and span multiple quality levels.

\paragraph{Training Setup and Hyperparameters}
We use the earlier-released Qwen2.5-7B-Instruct \cite{qwen2025qwen25technicalreport} as the backbone model to reduce potential data leakage, and adopt LoRA \cite{hu2022lora} for efficient training. The acceptance rate $\gamma$ is fixed at 31.4\%, which is the average acceptance rate for ICLR 2023 and 2024.

% Pujun: 防止泄露
\paragraph{Baselines}
We evaluate the following categories of baselines: \textbf{Classic GNNs}, including traditional GNN-based evaluation methods such as GCN~\cite{kipf2016semi}, GAT~\cite{velivckovic2017graph}, and GraphSAGE~\cite{hamilton2017inductive}.
\textbf{ LLMs}, including direct prompting approaches that ask LLMs from multiple providers to evaluate papers, with the main results reported for earlier-released models such as GPT-5-Mini, Gemini-2.5-Flash, and DeepSeek-V3.2 to reduce the risk of data leakage.
\textbf{Review Agents}, including agent-based evaluation systems such as AIScientist~\cite{lu2024ai}, AgentReview~\cite{jin2024agentreview}, SEA-7B~\cite{yu2024automated}, CycleReviewer-7B~\cite{weng2025cycleresearcher}, DeepReview-7B, and DeepReview-14B~\cite{zhu2025deepreview}.
\textbf{Comparative Systems}, including comparison-based evaluation methods such as PairReview~\cite{zhang2025from}, CNPE-7B~\cite{zheng2026isolated}, NAIP-8B~\cite{zhao2025words}, and NAIPv2-8B~\cite{zhao2025naipv2}.

\paragraph{Evaluation Metrics}
% Pujun: 无需介绍
% At the qualitative level, we examine whether each method covers the three key information sources emphasized in this work: Intrinsic Quality, Synchronic Links, and Diachronic Links. We also assess whether a method can generate evaluation text that is useful to reviewers and researchers. In addition, we consider the risk of data leakage, since LLMs trained on data collected after conference publication may already have been exposed to information about the papers.

% At the quantitative level, 
We use two groups of metrics to evaluate decision accuracy and ranking quality. The first formulates paper evaluation as a binary classification task that predicts whether a paper should be accepted, and we report Accuracy, F1, and AUC. The second measures the ability to rank high-quality papers ahead of low-quality ones, using Spearman's $\rho$, Kendall's $\tau$, and NDCG@10. We calculate 95\% confidence intervals and the statistical significance of differences using the bootstrap method.

\subsection{Main Results}

As shown in Table~\ref{tab:main}, classical GNN-based methods perform poorly overall. Although they are effective at modeling multi-source information with graph structures, they cannot generate textual feedback, which limits their practical value for reviewers and authors. More importantly, their architectures are not well suited to capturing the fine-grained semantics required for paper evaluation. In contrast, naïve LLMs achieve relatively strong performance on several metrics, consistent with their direct scoring paradigm that primarily relies on internal signals such as Intrinsic Quality. However, their reliability is undermined by the substantial risk of data leakage. The other two categories, review agents and comparison-based systems, are more competitive. Among review agents, DeepReview-14B trained with reinforcement learning performs best, achieving an average score of 0.5540 ($CI = [0.4950, 0.6136]$). Comparison-based methods place greater emphasis on modeling relational dependencies among contemporary papers and across time. NAIPv2 provides the strongest baseline performance, with an average score of 0.6087 ($CI = [0.5528, 0.6673]$).

Our GraphReview framework achieves an average score of 0.7893 ($CI = [0.7476, 0.8306]$) and consistently outperforms all baselines on both decision and ranking metrics. Compared with the strongest baseline, NAIPv2, our method achieves an average relative improvement of 29.7\% ($p<0.001$). On the core metrics, Accuracy and Spearman's $\rho$, the relative gains reach 23.7\% and 57.6\% (both $p<0.001$), respectively. These results demonstrate that our model is substantially more capable of both distinguishing accepted papers from rejected ones and predicting their relative ranking. Overall, the results validate the effectiveness of our framework, which is the only system that jointly integrates all information sources relevant to paper evaluation while also generating an advisory review for each paper.
% Our framework consistently outperforms all baselines on both decision and ranking metrics. Compared with the strongest baseline, NAIPv2, our method achieves an average relative improvement of 29.7\%. On the core metrics, Accuracy and Spearman's $\rho$, the relative gains reach 23.7\% and 57.6\%, respectively. Bootstrap confidence intervals on representative methods further indicate that the advantage is robust: GraphReview's overall performance interval is $[0.748,0.831]$, clearly separated from NAIPv2-8B's $[0.553,0.667]$. These results demonstrate that our model is substantially more capable of both distinguishing accepted papers from rejected ones and predicting their relative ranking. Overall, the results validate the effectiveness of our proposed framework, which is the only system that jointly integrates all information sources relevant to paper evaluation while also generating an advisory review for each paper.

\subsection{Ablation Study}

\begin{table}[t]
    \renewcommand{\arraystretch}{0.82}
	\setlength{\floatsep}{10pt} 
	\setlength{\textfloatsep}{10pt}
	\setlength{\aboverulesep}{1.5pt}
	\setlength{\belowrulesep}{1.5pt}
	\centering
	\scriptsize
	\setlength{\tabcolsep}{4pt}
	\begin{tabularx}{\columnwidth}{@{}
			>{\raggedright\arraybackslash}p{0.36\columnwidth}
			>{\centering\arraybackslash}X
			>{\centering\arraybackslash}X
			>{\centering\arraybackslash}X@{}}
		\toprule
		\textbf{Variant} & \textbf{Avg. Decision} & \textbf{Avg. Ranking} & \textbf{Avg. Performance} \\
		\specialrule{\lightrulewidth}{0.5pt}{0.5pt}
		\rowcolor{gray!10}
		\multicolumn{4}{@{}l}{\textbf{Information Sources}} \\
		\textcolor{red}{w/o } Graph  & 0.7640 & 0.5836 & 0.6738 \\
		\textcolor{red}{w/o } Intrinsic Quality & 0.8978 & 0.6127 & 0.7553 \\
		\textcolor{red}{w/o } Synchronic Links  & 0.8918 & \underline{0.6657} & \underline{0.7787} \\
		\textcolor{red}{w/o } Diachronic Links  & \underline{0.8979} & 0.6339 & 0.7659 \\
		\specialrule{\lightrulewidth}{0.5pt}{0.5pt}
		\rowcolor{gray!10}
		\multicolumn{4}{@{}l}{\textbf{Training Strategies}} \\
		\textcolor{red}{w/o } SFT  & \underline{0.8606} & \underline{0.6067} & \underline{0.7337} \\
		\textcolor{red}{w/o } RIML & 0.8604 & 0.5903 & 0.7254 \\
		\specialrule{\lightrulewidth}{0.5pt}{0.5pt}
		\rowcolor{gray!10}
		\multicolumn{4}{@{}l}{\textbf{Graph Construction}} \\
		Random Connection &  \underline{0.8731} &  \underline{0.6389} &  \underline{0.7560} \\
		\specialrule{\lightrulewidth}{0.5pt}{0.5pt}
		\rowcolor{gray!10}
		\multicolumn{4}{@{}l}{\textbf{Aggregation Alternatives}} \\
		Naïve Win-Rate Averaging & 0.8980 & 0.6461 & 0.7721 \\
		Bradley-Terry & \underline{0.9030} & \underline{0.6506} & \underline{0.7768} \\
		Borda Count & 0.8585 & 0.6357 & 0.7471 \\
		\specialrule{\lightrulewidth}{0.5pt}{0.5pt}
		\rowcolor{gray!10}
		\multicolumn{4}{@{}l}{\textbf{Full Model}} \\
		GraphReview      & \textbf{0.9125} & \textbf{0.6660} & \textbf{0.7893} \\
		\bottomrule
	\end{tabularx}
	\caption{Ablation study results of different information sources and training strategies. The full model always achieves the \textbf{best} results, and the \underline{second-best} results in each group are underlined.}
	\label{tab:ablation_all}
\end{table}

We conduct comprehensive ablation studies to assess the effects of information fusion, training strategies, graph construction, and aggregation methods. The results are summarized in Table~\ref{tab:ablation_all}.

We first assess the contribution of the three information sources. Excluding any source from graph construction reduces average performance. In particular, the full model outperforms the baseline that removes the graph structure and relies solely on point priors by 17.1\%, confirming the critical role of graph-based modeling.  Moreover, removing Intrinsic Quality, Diachronic Links, or Synchronic Links causes relative decreases of 4.3\%, 3.0\%, and 1.3\%, respectively. These results demonstrate that multi-source fusion effectively captures complementary signals of paper quality.

We next examine the training strategies. Removing the SFT stage reduces average performance by 7.0\%, while removing RIML leads to an 8.1\% decrease. This suggests that SFT introduces valuable review knowledge, whereas the soft-target objective of RIML further improves model alignment. The benefits of the two stages are complementary.

We then investigate the graph construction strategy. Under the same experimental setting, S2FM achieves 4.4\% higher average performance than random connectivity, indicating that semantically selected neighbors provide a more informative graph structure.

Finally, we compare PPR with alternative aggregation methods. PPR outperforms naïve win-rate averaging, Bradley-Terry modeling, and Borda count method. Overall, the full GraphReview model consistently achieves the best performance, confirming the effectiveness of its designs.

\subsection{Review Text Quality}

\begin{table}[t]
	\renewcommand{\arraystretch}{0.82}
	\setlength{\belowcaptionskip}{-5pt} 
	\setlength{\floatsep}{10pt} 
	\setlength{\textfloatsep}{10pt}
	\setlength{\aboverulesep}{1.5pt}
	\setlength{\belowrulesep}{1.5pt}
	\centering
	\scriptsize
	\setlength{\tabcolsep}{2.5pt}
	\begin{tabularx}{\columnwidth}{@{}
			>{\raggedright\arraybackslash}p{0.24\columnwidth}
			>{\centering\arraybackslash}X
			>{\centering\arraybackslash}X
			>{\centering\arraybackslash}X
			>{\centering\arraybackslash}X
			>{\centering\arraybackslash}X@{}}
		\toprule
		\textbf{\textcolor{gray}{vs} Baseline}
		& \textbf{Technical Depth}
		& \textbf{Evidence Grounding}
		& \textbf{Scientific Rigor}
		& \textbf{Revision Utility}
		& \textbf{Overall Preference} \\
		\specialrule{\lightrulewidth}{0.5pt}{0.5pt}
	
		\rowcolor{gray!10}
		\multicolumn{6}{@{}l}{\textbf{GraphReview }(Gemini-2.5-Flash) \textbf{vs Naïve LLM} (Gemini)} \\
		Gemini-2.5-Flash
		& 80.0\textcolor{gray!50}{/}\textcolor{gray}{19.5}
		& 63.5\textcolor{gray!50}{/}\textcolor{gray}{35.0}
		& 89.0\textcolor{gray!50}{/}\textcolor{gray}{9.5}
		& 93.5\textcolor{gray!50}{/}\textcolor{gray}{6.5}
		& 88.5\textcolor{gray!50}{/}\textcolor{gray}{11.5} \\
		Gemini-3-Flash
		& 70.0\textcolor{gray!50}{/}\textcolor{gray}{30.0}
		& 58.0\textcolor{gray!50}{/}\textcolor{gray}{41.0}
		& 82.0\textcolor{gray!50}{/}\textcolor{gray}{17.5}
		& 91.0\textcolor{gray!50}{/}\textcolor{gray}{8.0}
		& 82.0\textcolor{gray!50}{/}\textcolor{gray}{18.0} \\
		
		\specialrule{\lightrulewidth}{0.5pt}{0.5pt}
		
		\rowcolor{gray!10}
		\multicolumn{6}{@{}l}{\textbf{GraphReview }(DeepSeek-V3.2) \textbf{vs Naïve LLM} (DeepSeek)} \\
		DeepSeek-V3.2
		& 76.0\textcolor{gray!50}{/}\textcolor{gray}{24.0}
		& 63.0\textcolor{gray!50}{/}\textcolor{gray}{35.0}
		& 83.5\textcolor{gray!50}{/}\textcolor{gray}{14.0}
		& 95.5\textcolor{gray!50}{/}\textcolor{gray}{3.0}
		& 84.5\textcolor{gray!50}{/}\textcolor{gray}{15.5} \\
		DeepSeek-V4-Pro
		& 86.0\textcolor{gray!50}{/}\textcolor{gray}{13.0}
		& 78.0\textcolor{gray!50}{/}\textcolor{gray}{21.0}
		& 85.5\textcolor{gray!50}{/}\textcolor{gray}{13.5}
		& 95.0\textcolor{gray!50}{/}\textcolor{gray}{4.5}
		& 85.5\textcolor{gray!50}{/}\textcolor{gray}{14.0} \\
		
		\specialrule{\lightrulewidth}{0.5pt}{0.5pt}
		\rowcolor{gray!10}
		\multicolumn{6}{@{}l}{\textbf{GraphReview }(DeepSeek-V3.2) \textbf{vs Other Baselines}} \\
		DeepReview-7B
		& 92.5\textcolor{gray!50}{/}\textcolor{gray}{7.5}
		& 95.0\textcolor{gray!50}{/}\textcolor{gray}{5.0}
		& 97.0\textcolor{gray!50}{/}\textcolor{gray}{3.0}
		& 99.0\textcolor{gray!50}{/}\textcolor{gray}{1.0}
		& 97.5\textcolor{gray!50}{/}\textcolor{gray}{2.5} \\
		DeepReview-14B
		& 63.5\textcolor{gray!50}{/}\textcolor{gray}{36.5}
		& 66.5\textcolor{gray!50}{/}\textcolor{gray}{33.0}
		& 76.5\textcolor{gray!50}{/}\textcolor{gray}{23.5}
		& 63.0\textcolor{gray!50}{/}\textcolor{gray}{35.0}
		& 69.0\textcolor{gray!50}{/}\textcolor{gray}{31.0} \\
		CNPE-7B
		& 91.0\textcolor{gray!50}{/}\textcolor{gray}{8.0}
		& 97.0\textcolor{gray!50}{/}\textcolor{gray}{2.5}
		& 97.0\textcolor{gray!50}{/}\textcolor{gray}{2.5}
		& 99.0\textcolor{gray!50}{/}\textcolor{gray}{0.5}
		& 97.0\textcolor{gray!50}{/}\textcolor{gray}{2.5} \\
		\bottomrule
	\end{tabularx}
	\caption{Win/Loss rates (\%, Ties are omitted) of our review texts with different consolidation model against same-family naïve LLMs, and other baselines.}
	\label{tab:text_winrates}
\end{table}

% old: We systematically compare the review quality of our method with representative baselines. We randomly sample 200 papers and use LLM-as-a-Judge to evaluate the aggregated review texts. As shown in Table \ref{tab:text_winrates}, each comparison covers technical depth, evidence grounding, scientific rigor, revision utility, and overall preference. Our method achieves a win rate above 50\% against all baselines, indicating the highest overall review text quality.

We compare the review quality of GraphReview against baselines. We randomly sample 200 papers and adopt an LLM-as-a-Judge protocol to evaluate the aggregated reviews, with response order randomized to mitigate positional bias. GPT-5.4 serves as the evaluator, while all evaluated models belong to different model families from the evaluator to reduce potential bias. As shown in Table \ref{tab:text_winrates}, each pairwise comparison assesses technical depth, evidence grounding, scientific rigor, revision utility, and overall preference. 

The within-family evaluation shows that GraphReview substantially outperforms naïve LLM-based review generation using the same backbone, achieving overall preference win rates of 88.5\% with Gemini-2.5-Flash and 84.5\% with DeepSeek-V3.2. Replacing the naïve baselines with stronger next-generation models has only a limited impact on the win rates (Gemini: -6.5; DeepSeek: +1.0), suggesting that the improvements mainly arise from graph-based structured information fusion rather than increased backbone quality. 

GraphReview also outperforms other baselines across all evaluation dimensions, achieving overall preference win rates above 50\%.

\subsection{Hyperparameter Analysis}

\begin{figure}[htbp]
	\setlength{\belowcaptionskip}{-5pt} 
	\setlength{\floatsep}{10pt} 
	\setlength{\textfloatsep}{10pt}
	\centering
	\includegraphics[width=\columnwidth]{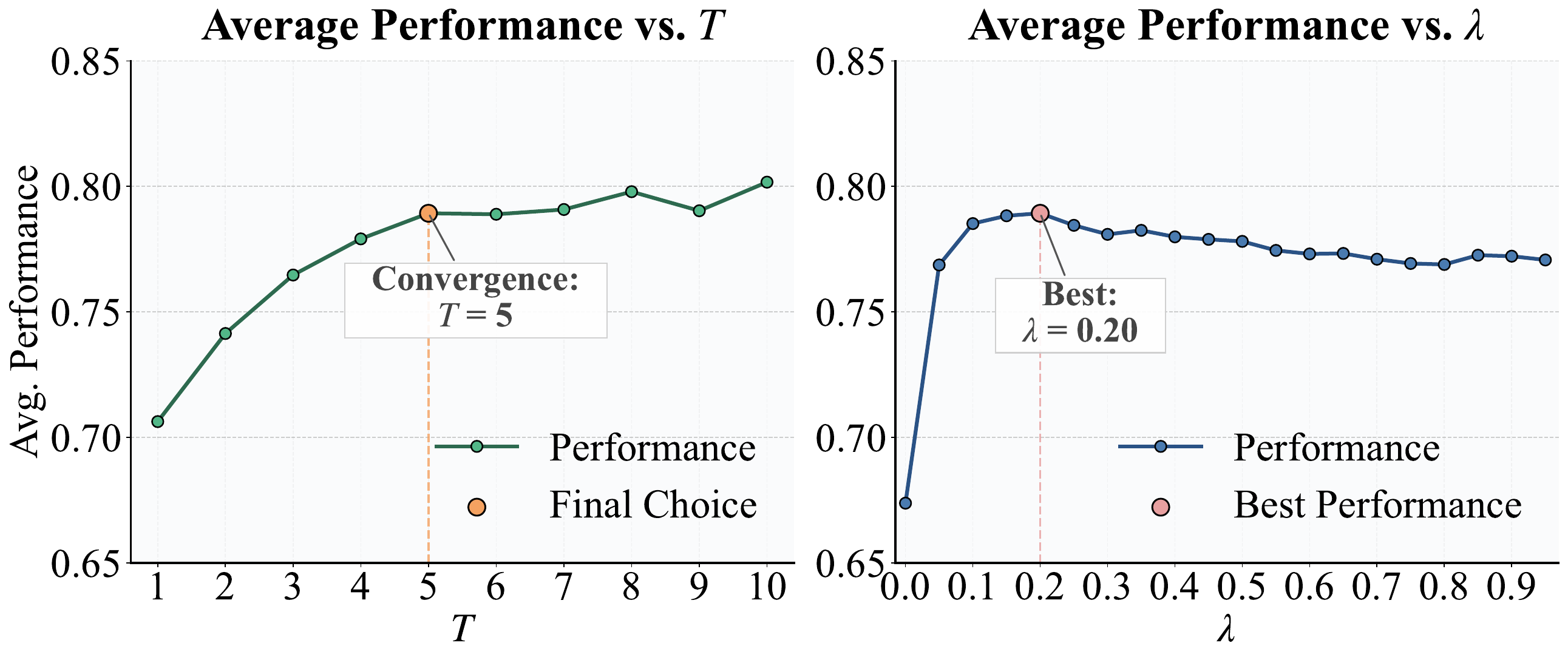}
	\caption{Hyperparameter analysis, including the convergence of model performance as $T$ increases (left) and the selection of the optimal damping factor $\lambda$ (right).}
	\label{fig:hyperparameter_combined}
\end{figure}

% old: We further analyze the key hyperparameters in Figure~\ref{fig:hyperparameter_combined}. We first examine how overall performance changes as $T$ increases. The results exhibit a clear convergence pattern: performance improves steadily in the first few iterations, while the gains become marginal after $T>5$. This observation indicates that $T=5$ achieves a favorable balance between effectiveness and efficiency.
We further analyze the key hyperparameters in Figure~\ref{fig:hyperparameter_combined}. We first examine how overall performance changes as $T$ increases. Here, $T$ controls inference-time expansion by activating more S2FM semantic edges and therefore increasing the amount of structured graph evidence available for aggregation. The results exhibit a clear convergence pattern: performance improves steadily in the first few iterations, while the gains become marginal after $T>5$. This observation indicates that $T=5$ achieves a favorable balance between effectiveness and efficiency. 

We then examine the effect of $\lambda$. Setting $\lambda=0$ corresponds to relying solely on point priors. Incorporating graph information consistently yields strong performance, with limited sensitivity to $\lambda$. The best result is achieved at $\lambda=0.20$. These results validate the effectiveness of graph-based evidence aggregation and show that moderate relational evidence can substantially refine the initial estimates while remaining robust across different values of $\lambda$.

% Pujun: 这里修改弱传播的说法，这个说法很容易引起误解，要参考rebuttal的内容

% old: We then investigate the effect of $\lambda$. Unlike conventional PageRank, which typically prefers a larger damping factor, PPR achieves the best performance at $\lambda=0.20$. This finding suggests that weaker graph propagation is better suited to our task, as a smaller $\lambda$ reduces error accumulation and mitigates oversmoothing.
% We then investigate the effect of $\lambda$, which controls the contribution of graph propagation relative to the initial quality estimates. Incorporating graph evidence with $\lambda=0.20$ yields the best average performance of 0.7893. These results indicate that moderate graph propagation effectively complements the initial estimates while preserving their reliability.

\subsection{Generalization}

% old: content...
% old: \end{asy}能力要超过这个模型。可以考虑加入未训练/随机模型的 baseline。

\begin{figure}[htbp]
	\setlength{\belowcaptionskip}{-5pt} 
	\setlength{\floatsep}{10pt} 
	\setlength{\textfloatsep}{10pt}
	\centering
	\includegraphics[width=\columnwidth]{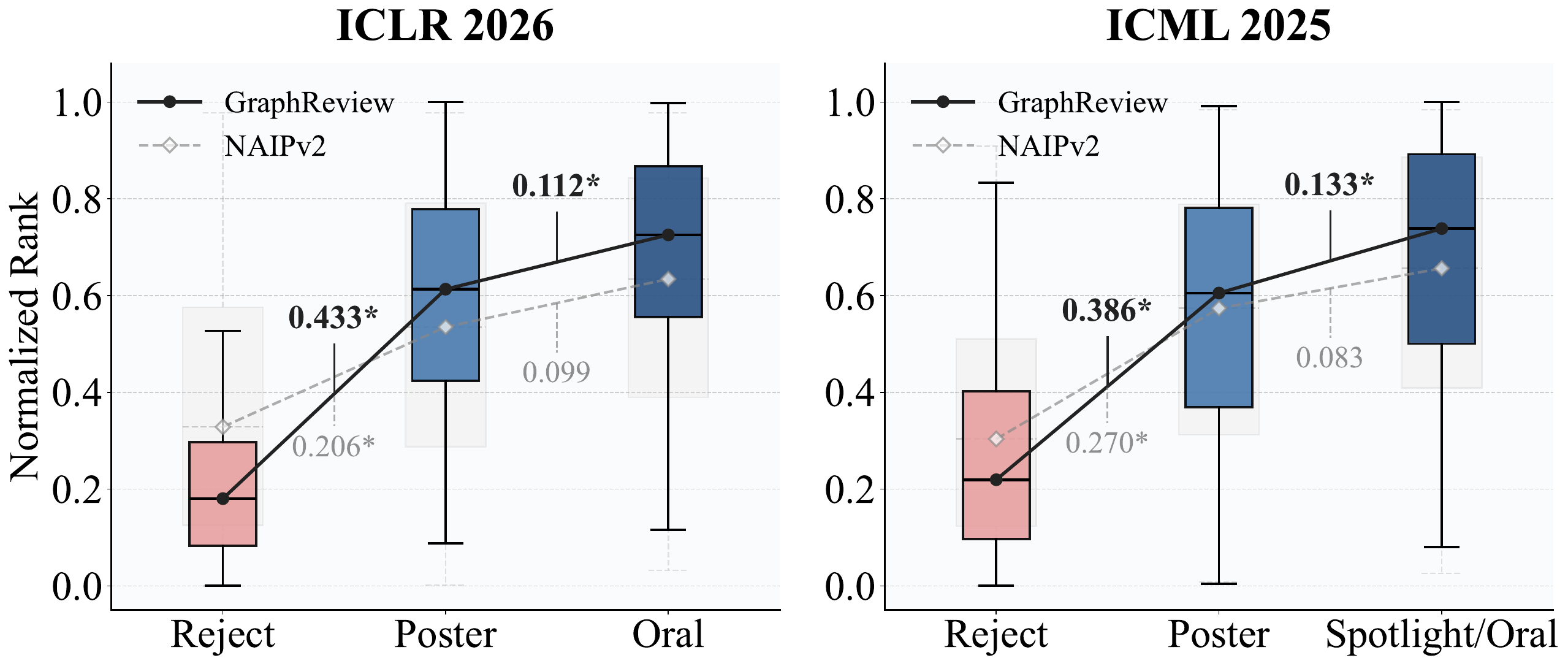}
	\caption{Generalization results. The y-axis shows normalized rank from 0 to 1, where higher values indicate higher predicted paper quality. The annotated values denote the median difference between groups, and $*$ indicates statistical significance with $p<0.001$.}
	\label{fig:boxplot_combined}
\end{figure}

% old: We evaluate generalization in two settings. The cross-time-period setting is based on ICLR 2026 and includes three categories: Rejected, Poster, and Oral. The cross-venue setting is based on ICML 2025 and includes Rejected, Poster, and Spotlight/Oral. For each conference, we randomly sample 500 papers while keeping the class distribution as balanced as possible. As percentile ranks do not follow a normal distribution, we use the nonparametric Mann-Whitney U test.
We evaluate generalization in two settings. The cross-time-period setting is based on ICLR 2026 and includes three categories: Rejected, Poster, and Oral. The cross-venue setting is based on ICML 2025 and includes Rejected, Poster, and Spotlight/Oral. For each conference, we randomly sample 500 papers while keeping the class distribution as balanced as possible. As percentile ranks do not follow a normal distribution, we use the nonparametric Mann-Whitney U test and compare GraphReview with the strongest baseline, NAIPv2-8B, under the same setting.

% old: The results show significant differences across all categories, with the same pattern appearing in both datasets: the gap between rejected and poster papers is large, at 0.433 for ICLR 2026 and 0.336 for ICML 2025, while the gap within accepted papers is smaller, at 0.112 and 0.133. This pattern matches review practice, in which accepted and rejected papers usually differ clearly in quality and differences within accepted papers are typically smaller, suggesting that our method captures both broad and fine-grained quality differences.
The results show significant differences across all categories, with a consistent trend across both datasets. The gap between rejected and poster papers is substantial, reaching 0.433 on ICLR 2026 and 0.386 on ICML 2025. In contrast, the gap among accepted papers is smaller, at 0.112 and 0.133, respectively. This pattern is consistent with review practice: accepted and rejected papers usually differ clearly in quality, whereas quality differences among accepted papers are more subtle. These results indicate that our method captures both coarse-grained and fine-grained distinctions in paper quality.

Compared with NAIPv2-8B, GraphReview increases the Reject-vs-Poster gap by 0.227 on ICLR 2026 and by 0.116 on ICML 2025. It also increases the accepted-paper gap by 0.013 and 0.050 on the two datasets, respectively. Moreover, unlike GraphReview, NAIPv2-8B does not yield a significant Poster-vs-Spotlight/Oral gap ($p>0.001$). These results suggest that GraphReview generalizes more effectively across time periods and venues when identifying differences in paper quality.

\section{Conclusion}

% 中间这一段用于调整长度
We propose GraphReview, an LLM-based graph framework for scientific paper evaluation that integrates complementary evidence through inference-time graph evidence expansion. 
% Unlike existing methods that fail to capture the structure of multi-source evidence, GraphReview supports the integration of diverse signals.
Experiments show that it achieves leading performance on quality-ranking, decision-prediction, and review-generation benchmarks, while demonstrating strong generalization.
% Overall, GraphReview provides a data-driven and holistic framework for leveraging graph-structured evidence in LLM-based paper evaluation, representing a meaningful step toward automated scientific paper evaluation.
Overall, GraphReview provides a data-driven and holistic framework for leveraging graph-structured review evidence, representing a meaningful step toward automated scientific paper evaluation.

\bibliography{aaai2027}

@misc{alberts2008reviewing,
	title={Reviewing peer review},
	author={Alberts, Bruce and Hanson, Brooks and Kelner, Katrina L},
	journal={Science},
	volume={321},
	number={5885},
	pages={15--15},
	year={2008},
	publisher={American Association for the Advancement of Science}
}

@article{he2023research,
	title={Research explosion: More effort to climb onto shoulders of the giant},
	author={He, Guoxiu and Sun, Aixin and Lu, Wei},
	journal={arXiv preprint arXiv:2307.06506},
	year={2023}
}

@inproceedings{zhou2024llm,
	title={Is LLM a reliable reviewer? a comprehensive evaluation of LLM on automatic paper reviewing tasks},
	author={Zhou, Ruiyang and Chen, Lu and Yu, Kai},
	booktitle={Proceedings of the 2024 joint international conference on computational linguistics, language resources and evaluation (LREC-COLING 2024)},
	pages={9340--9351},
	year={2024}
}

@inproceedings{du2024llms,
	title={LLMs assist NLP researchers: Critique paper (meta-) reviewing},
	author={Du, Jiangshu and Wang, Yibo and Zhao, Wenting and Deng, Zhongfen and Liu, Shuaiqi and Lou, Renze and Zou, Henry Peng and Venkit, Pranav Narayanan and Zhang, Nan and Srinath, Mukund and others},
	booktitle={Proceedings of the 2024 conference on empirical methods in natural language processing},
	pages={5081--5099},
	year={2024}
}

@article{zhuang2025large,
	title={Large language models for automated scholarly paper review: A survey},
	author={Zhuang, Zhenzhen and Chen, Jiandong and Xu, Hongfeng and Jiang, Yuwen and Lin, Jialiang},
	journal={Information Fusion},
	pages={103332},
	year={2025},
	publisher={Elsevier}
}

@article{thakkar2025can,
	title={Can llm feedback enhance review quality? a randomized study of 20k reviews at iclr 2025},
	author={Thakkar, Nitya and Yuksekgonul, Mert and Silberg, Jake and Garg, Animesh and Peng, Nanyun and Sha, Fei and Yu, Rose and Vondrick, Carl and Zou, James},
	journal={arXiv preprint arXiv:2504.09737},
	year={2025}
}

@article{latona2024ai,
	title={The ai review lottery: Widespread ai-assisted peer reviews boost paper scores and acceptance rates},
	author={Latona, Giuseppe Russo and Ribeiro, Manoel Horta and Davidson, Tim R and Veselovsky, Veniamin and West, Robert},
	journal={arXiv preprint arXiv:2405.02150},
	year={2024}
}

@article{li2025can,
	title={Can Large Language Models Be Trusted Paper Reviewers? A Feasibility Study},
	author={Li, Chuanlei and Hu, Xu and Xu, Minghui and Li, Kun and Zhang, Yue and Cheng, Xiuzhen},
	journal={arXiv preprint arXiv:2506.17311},
	year={2025}
}

@article{thakkar2026large,
	title={A large-scale randomized study of large language model feedback in peer review},
	author={Thakkar, Nitya and Yuksekgonul, Mert and Silberg, Jake and Garg, Animesh and Peng, Nanyun and Sha, Fei and Yu, Rose and Vondrick, Carl and Zou, James},
	journal={Nature Machine Intelligence},
	pages={1--11},
	year={2026},
	publisher={Nature Publishing Group UK London}
}

@article{wu2019large,
	title={Large teams develop and small teams disrupt science and technology},
	author={Wu, Lingfei and Wang, Dashun and Evans, James A},
	journal={Nature},
	volume={566},
	number={7744},
	pages={378--382},
	year={2019},
	publisher={Nature Publishing Group UK London}
}

@article{fortunato2018science,
	title={Science of science},
	author={Fortunato, Santo and Bergstrom, Carl T and B{\"o}rner, Katy and Evans, James A and Helbing, Dirk and Milojevi{\'c}, Sta{\v{s}}a and Petersen, Alexander M and Radicchi, Filippo and Sinatra, Roberta and Uzzi, Brian and others},
	journal={Science},
	volume={359},
	number={6379},
	pages={eaao0185},
	year={2018},
	publisher={American Association for the Advancement of Science}
}

@inproceedings{gilmer2017neural,
	title={Neural message passing for quantum chemistry},
	author={Gilmer, Justin and Schoenholz, Samuel S and Riley, Patrick F and Vinyals, Oriol and Dahl, George E},
	booktitle={International conference on machine learning},
	pages={1263--1272},
	year={2017},
	organization={Pmlr}
}

@inproceedings{jeh2003scaling,
	title={Scaling personalized web search},
	author={Jeh, Glen and Widom, Jennifer},
	booktitle={Proceedings of the 12th international conference on World Wide Web},
	pages={271--279},
	year={2003}
}

@inproceedings{jin2024agentreview,
	title={AgentReview: Exploring Peer Review Dynamics with LLM Agents},
	author={Jin, Yiqiao and Zhao, Qinlin and Wang, Yiyang and Chen, Hao and Zhu, Kaijie and Xiao, Yijia and Wang, Jindong},
	booktitle={Proceedings of the 2024 Conference on Empirical Methods in Natural Language Processing},
	pages={1208--1226},
	year={2024}
}

@article{lu2024ai,
	title={The ai scientist: Towards fully automated open-ended scientific discovery},
	author={Lu, Chris and Lu, Cong and Lange, Robert Tjarko and Foerster, Jakob and Clune, Jeff and Ha, David},
	journal={arXiv preprint arXiv:2408.06292},
	year={2024}
}

@inproceedings{weng2025cycleresearcher,
	title={CycleResearcher: Improving Automated Research via Automated Review},
	author={Weng, Yixuan and Zhu, Minjun and Bao, Guangsheng and Zhang, Hongbo and Wang, Jindong and Zhang, Yue and Yang, Linyi},
	booktitle={The Thirteenth International Conference on Learning Representations},
	year={2025}
}

@article{tan2024peer,
	title={Peer review as a multi-turn and long-context dialogue with role-based interactions},
	author={Tan, Cheng and Lyu, Dongxin and Li, Siyuan and Gao, Zhangyang and Wei, Jingxuan and Ma, Siqi and Liu, Zicheng and Li, Stan Z},
	journal={arXiv preprint arXiv:2406.05688},
	year={2024}
}

@inproceedings{zeng2025reviewrl,
	title={ReviewRL: Towards Automated Scientific Review with RL},
	author={Zeng, Sihang and Tian, Kai and Zhang, Kaiyan and Wang, Yuru and Gao, Junqi and Liu, Runze and Yang, Sa and Li, Jingxuan and Long, Xinwei and Ma, Jiaheng and others},
	booktitle={Proceedings of the 2025 Conference on Empirical Methods in Natural Language Processing},
	pages={16942--16954},
	year={2025}
}

@inproceedings{yu2024automated,
	title={Automated Peer Reviewing in Paper SEA: Standardization, Evaluation, and Analysis},
	author={Yu, Jianxiang and Ding, Zichen and Tan, Jiaqi and Luo, Kangyang and Weng, Zhenmin and Gong, Chenghua and Zeng, Long and Cui, RenJing and Han, Chengcheng and Sun, Qiushi and others},
	booktitle={Findings of the Association for Computational Linguistics: EMNLP 2024},
	pages={10164--10184},
	year={2024}
}

@article{tyser2024ai,
	title={Ai-driven review systems: evaluating llms in scalable and bias-aware academic reviews},
	author={Tyser, Keith and Segev, Ben and Longhitano, Gaston and Zhang, Xin-Yu and Meeks, Zachary and Lee, Jason and Garg, Uday and Belsten, Nicholas and Shporer, Avi and Udell, Madeleine and others},
	journal={arXiv preprint arXiv:2408.10365},
	year={2024}
}

@article{garg2025revieweval,
	title={ReviewEval: An Evaluation Framework for AI-Generated Reviews},
	author={Garg, Madhav Krishan and Prasad, Tejash and Singhal, Tanmay and Kirtani, Chhavi and Mandal, Murari and Kumar, Dhruv},
	journal={arXiv preprint arXiv:2502.11736},
	year={2025}
}

@inproceedings{chang2025treereview,
	title={TreeReview: A dynamic tree of questions framework for deep and efficient LLM-based scientific peer review},
	author={Chang, Yuan and Li, Ziyue and Zhang, Hengyuan and Kong, Yuanbo and Wu, Yanru and So, Hayden Kwok-Hay and Guo, Zhijiang and Zhu, Liya and Wong, Ngai},
	booktitle={Proceedings of the 2025 Conference on Empirical Methods in Natural Language Processing},
	pages={15662--15693},
	year={2025}
}

@inproceedings{zhu2025deepreview,
	title={DeepReview: Improving LLM-based Paper Review with Human-like Deep Thinking Process},
	author={Zhu, Minjun and Weng, Yixuan and Yang, Linyi and Zhang, Yue},
	booktitle={Proceedings of the 63rd Annual Meeting of the Association for Computational Linguistics (Volume 1: Long Papers)},
	pages = {29330--29355},
	year={2025}
}

@inproceedings{lu2025agent,
	title={Agent Reviewers: Domain-specific Multimodal Agents with Shared Memory for Paper Review},
	author={Lu, Kai and Xu, Shixiong and Li, Jinqiu and Ding, Kun and Meng, Gaofeng},
	booktitle={Forty-second International Conference on Machine Learning},
	pages = {40803--40830},
	year={2025}
}

@inproceedings{zhang2025from,
	title={From Replication to Redesign: Exploring Pairwise Comparisons for {LLM}-Based Peer Review},
	author={Yaohui Zhang and Haijing ZHANG and Wenlong Ji and Tianyu Hua and Nick Haber and Hancheng Cao and Weixin Liang},
	booktitle={The Thirty-ninth Annual Conference on Neural Information Processing Systems},
	year={2025}
}

@article{hopner2025automatic,
	title={Automatic Evaluation Metrics for Artificially Generated Scientific Research},
	author={H{\"o}pner, Niklas and Eshuijs, Leon and Alivanistos, Dimitrios and Zamprogno, Giacomo and Tiddi, Ilaria},
	journal={arXiv preprint arXiv:2503.05712},
	year={2025}
}

@inproceedings{zhao2025words,
	title={From Words to Worth: Newborn Article Impact Prediction with LLM},
	author={Zhao, Penghai and Xing, Qinghua and Dou, Kairan and Tian, Jinyu and Tai, Ying and Yang, Jian and Cheng, Ming-Ming and Li, Xiang},
	booktitle={Proceedings of the AAAI Conference on Artificial Intelligence},
	volume={39},
	number={1},
	pages={1183--1191},
	year={2025}
}

@article{he2023h2cgl,
	title={H2CGL: Modeling dynamics of citation network for impact prediction},
	author={He, Guoxiu and Xue, Zhikai and Jiang, Zhuoren and Kang, Yangyang and Zhao, Star and Lu, Wei},
	journal={Information Processing \& Management},
	volume={60},
	number={6},
	pages={103512},
	year={2023},
	publisher={Elsevier}
}

@inproceedings{xue2024predicting,
	title={Predicting scientific impact through diffusion, conformity, and contribution disentanglement},
	author={Xue, Zhikai and He, Guoxiu and Jiang, Zhuoren and Gu, Sichen and Kang, Yangyang and Zhao, Star and Lu, Wei},
	booktitle={Proceedings of the 33rd ACM International Conference on Information and Knowledge Management},
	pages={2764--2774},
	year={2024}
}

@article{zheng2026isolated,
	title={From Isolated Scoring to Collaborative Ranking: A Comparison-Native Framework for LLM-Based Paper Evaluation},
	author={Zheng, Pujun and Yao, Jiacheng and Zheng, Jinquan and Gu, Chenyang and He, Guoxiu and Liu, Jiawei and Huang, Yong and Guo, Tianrui and Lu, Wei},
	journal={arXiv preprint arXiv:2603.17588},
	year={2026}
}

@article{zhao2025naipv2,
	title={NAIPv2: Debiased Pairwise Learning for Efficient Paper Quality Estimation},
	author={Zhao, Penghai and Tian, Jinyu and Xing, Qinghua and Zhang, Xin and Li, Zheng and Qian, Jianjun and Cheng, Ming-Ming and Li, Xiang},
	journal={arXiv preprint arXiv:2509.25179},
	year={2025}
}

@article{velivckovic2017graph,
	title={Graph attention networks},
	author={Veli{\v{c}}kovi{\'c}, Petar and Cucurull, Guillem and Casanova, Arantxa and Romero, Adriana and Lio, Pietro and Bengio, Yoshua},
	journal={arXiv preprint arXiv:1710.10903},
	year={2017}
}

@article{hamilton2017inductive,
	title={Inductive representation learning on large graphs},
	author={Hamilton, Will and Ying, Zhitao and Leskovec, Jure},
	journal={Advances in neural information processing systems},
	volume={30},
	year={2017}
}

@article{kipf2016semi,
	title={Semi-supervised classification with graph convolutional networks},
	author={Kipf, Thomas N and Welling, Max},
	journal={arXiv preprint arXiv:1609.02907},
	year={2016}
}

@article{tao2025code,
	title={Code graph model (cgm): A graph-integrated large language model for repository-level software engineering tasks},
	author={Tao, Hongyuan and Zhang, Ying and Tang, Zhenhao and Peng, Hongen and Zhu, Xukun and Liu, Bingchang and Yang, Yingguang and Zhang, Ziyin and Xu, Zhaogui and Zhang, Haipeng and others},
	journal={arXiv preprint arXiv:2505.16901},
	year={2025}
}

@article{luo2025g,
	title={G-reasoner: Foundation Models for Unified Reasoning over Graph-structured Knowledge},
	author={Luo, Linhao and Zhao, Zicheng and Liu, Junnan and Qiu, Zhangchi and Dong, Junnan and Panev, Serge and Gong, Chen and Vu, Thuy-Trang and Haffari, Gholamreza and Phung, Dinh and others},
	journal={arXiv preprint arXiv:2509.24276},
	year={2025}
}

@article{jing2026entropy,
	title={Entropy-Guided Dynamic Tokens for Graph-LLM Alignment in Molecular Understanding},
	author={Jing, Zihao and Zeng, Qiuhao and Fang, Ruiyi and Sun, Yan and Wang, Boyu and Hu, Pingzhao},
	journal={arXiv preprint arXiv:2602.02742},
	year={2026}
}

@article{dong2025youtu,
	title={Youtu-graphrag: Vertically unified agents for graph retrieval-augmented complex reasoning},
	author={Dong, Junnan and An, Siyu and Yu, Yifei and Zhang, Qian-Wen and Luo, Linhao and Huang, Xiao and Wu, Yunsheng and Yin, Di and Sun, Xing},
	journal={arXiv preprint arXiv:2508.19855},
	year={2025}
}

@article{finkelshtein2025actions,
	title={Actions Speak Louder than Prompts: A Large-Scale Study of LLMs for Graph Inference},
	author={Finkelshtein, Ben and Cucerzan, Silviu and Jauhar, Sujay Kumar and White, Ryen},
	journal={arXiv preprint arXiv:2509.18487},
	year={2025}
}

@article{zhang2023graph,
	title={Graph-toolformer: To empower llms with graph reasoning ability via prompt augmented by chatgpt},
	author={Zhang, Jiawei},
	journal={arXiv preprint arXiv:2304.11116},
	year={2023}
}

@inproceedings{tang2024graphgpt,
	title={Graphgpt: Graph instruction tuning for large language models},
	author={Tang, Jiabin and Yang, Yuhao and Wei, Wei and Shi, Lei and Su, Lixin and Cheng, Suqi and Yin, Dawei and Huang, Chao},
	booktitle={Proceedings of the 47th International ACM SIGIR Conference on Research and Development in Information Retrieval},
	pages={491--500},
	year={2024}
}

@article{gutierrez2024hipporag,
	title={Hipporag: Neurobiologically inspired long-term memory for large language models},
	author={Guti{\'e}rrez, Bernal J and Shu, Yiheng and Gu, Yu and Yasunaga, Michihiro and Su, Yu},
	journal={Advances in neural information processing systems},
	volume={37},
	pages={59532--59569},
	year={2024}
}

@article{guo2024lightrag,
	title={Lightrag: Simple and fast retrieval-augmented generation},
	author={Guo, Zirui and Xia, Lianghao and Yu, Yanhua and Ao, Tian and Huang, Chao},
	journal={arXiv preprint arXiv:2410.05779},
	volume={2},
	number={3},
	year={2024}
}

@article{edge2024local,
	title={From local to global: A graph rag approach to query-focused summarization},
	author={Edge, Darren and Trinh, Ha and Cheng, Newman and Bradley, Joshua and Chao, Alex and Mody, Apurva and Truitt, Steven and Metropolitansky, Dasha and Ness, Robert Osazuwa and Larson, Jonathan},
	journal={arXiv preprint arXiv:2404.16130},
	year={2024}
}

@article{zhao2023graphtext,
	title={Graphtext: Graph reasoning in text space},
	author={Zhao, Jianan and Zhuo, Le and Shen, Yikang and Qu, Meng and Liu, Kai and Bronstein, Michael and Zhu, Zhaocheng and Tang, Jian},
	journal={arXiv preprint arXiv:2310.01089},
	year={2023}
}

@article{wang2023can,
	title={Can language models solve graph problems in natural language?},
	author={Wang, Heng and Feng, Shangbin and He, Tianxing and Tan, Zhaoxuan and Han, Xiaochuang and Tsvetkov, Yulia},
	journal={Advances in Neural Information Processing Systems},
	volume={36},
	pages={30840--30861},
	year={2023}
}

@article{wang2025unigte,
	title={UniGTE: Unified Graph-Text Encoding for Zero-Shot Generalization across Graph Tasks and Domains},
	author={Wang, Duo and Zuo, Yuan and Lu, Guangyue and Wu, Junjie},
	journal={arXiv preprint arXiv:2510.16885},
	year={2025}
}

@article{joshi2025transformers,
	title={Transformers are graph neural networks},
	author={Joshi, Chaitanya K},
	journal={arXiv preprint arXiv:2506.22084},
	year={2025}
}

@article{frasca2025neural,
	title={Neural Message-Passing on Attention Graphs for Hallucination Detection},
	author={Frasca, Fabrizio and Bar-Shalom, Guy and Ziser, Yftah and Maron, Haggai},
	journal={arXiv preprint arXiv:2509.24770},
	year={2025}
}

@article{qwen2025qwen25technicalreport,
	title={Qwen2.5 Technical Report}, 
	author = {{Qwen Team}},
	journal={arXiv preprint arXiv:2412.15115}, 
	year={2024},
}

@inproceedings{hu2022lora,
	title={LoRA: Low-Rank Adaptation of Large Language Models},
	author={Hu, Edward J and Wallis, Phillip and Allen-Zhu, Zeyuan and Li, Yuanzhi and Wang, Shean and Wang, Lu and Chen, Weizhu and others},
	booktitle={International Conference on Learning Representations},
	year={2022}
}

@inproceedings{zheng2026navigating,
	title={Navigating Through Paper Flood: Advancing LLM-Based Paper Evaluation Through Domain-Aware Retrieval and Latent Reasoning},
	author={Zheng, Wuqiang and Xu, Yiyan and Lin, Xinyu and Gao, Chongming and Wang, Wenjie and Feng, Fuli},
	booktitle={Proceedings of the AAAI Conference on Artificial Intelligence},
	volume={40},
	number={41},
	pages={35041--35049},
	year={2026}
}

@inproceedings{feng2025grapheval,
	author = {Feng, Tao and Sun, Yihang and You, Jiaxuan},
	booktitle = {International Conference on Learning Representations},
	editor = {Y. Yue and A. Garg and N. Peng and F. Sha and R. Yu},
	pages = {99181--99209},
	title = {GraphEval: A Lightweight Graph-Based LLM Framework for Idea Evaluation},
	url = {https://proceedings.iclr.cc/paper_files/paper/2025/file/f5ce40ee957e4f76ef53c09d0bae20f4-Paper-Conference.pdf},
	volume = {2025},
	year = {2025}
}

\appendix
\label{sec:appendix}
\newpage
\section{Use of Generative AI}

Generative artificial intelligence was used in this study only for language support, including grammar correction, phrasing refinement, and style standardization. It was not involved in the paper's conceptualization, scientific analysis, or substantive writing. All AI-assisted content was reviewed, verified, and, where necessary, revised by the authors in accordance with ethical standards. The authors take full responsibility for the paper's content and conclusions.

\section{Theoretical Implications}

A growing line of work in automated paper evaluation seeks to mimic the human review process. We take a different view. In our opinion, \textit{the more closely a method reproduces the procedure by which humans review papers, the further it may be from an ideal decision process}. Human reviewing involves rich internal reasoning and highly personal judgments that are not reflected in the final written review. A strong review system should therefore learn from the evidence-based reasoning underlying human judgment, rather than from the surface form of review texts.

Therefore, AI-based paper evaluation does not need to follow the same paradigm as human reviewing. Machines offer distinct advantages. Unlike human reviewers, they can draw on much larger volumes of data and perform fine-grained comparisons between a paper and a wide range of relevant evidence. This differs fundamentally from human practice. Due to limited time, human reviewers often cannot make such extensive and detailed empirical connections, and instead rely on broad impressions formed through prior experience in the field. This is exactly why we introduce graph structure into the review process: to make the evidence chain explicit while exploiting the large-scale processing capabilities of machines. At the same time, humans retain unique strengths, particularly intuition for innovation and for identifying important problems. Such intuition is grounded in embodied experience and creativity. Current AI systems do not yet possess this ability, and closing this gap should remain an important direction for future research.

In the long run, AI-assisted reviewing should move beyond imitating human review and instead develop into a review paradigm with its own strengths: comprehensive, data-intensive, and systems-oriented. Such a paradigm can complement human capabilities more effectively in a future of human-AI collaboration. From this perspective, our work represents a meaningful step toward that goal.

\section{Full Literature Review}

\subsection{Automated Paper Review}

The most common line of work evaluates a paper's inherent attributes, using LLM-driven agent review frameworks to simulate and reproduce the peer-review process with automated pointwise scoring \cite{jin2024agentreview,lu2024ai}, often augmented by iterative, multi-round conversational reviewing \cite{weng2025cycleresearcher,tan2024peer}. Many approaches further improve performance through training, transforming open-source models into expert reviewers via domain-specific fine-tuning or reinforcement learning \cite{zeng2025reviewrl}. For example, \citet{yu2024automated} fine-tunes models to generate customized feedback aligned with human preferences \cite{tyser2024ai,garg2025revieweval}. Other methods employ agents to construct tree-structured workflows to enhance generation quality \cite{chang2025treereview}. More recent systems combine model training with agent techniques, such as integrating evidence-based reasoning and structured analysis into multi-stage workflows \cite{zhu2025deepreview}, or building multi-agent systems with prompt engineering, shared memory, and multimodal perception \cite{lu2025agent}.

However, recent studies suggest that pointwise scoring is sensitive to inconsistencies in rating scales across venues and time. As a result, many works shift to parallel submissions within the same venue and develop comparison-based scoring strategies. Examples include predicting paper quality scores via pairwise comparisons \cite{zhao2025naipv2}, performing pairwise comparisons and aggregating preferences to derive a global ranking \cite{zhang2025from,hopner2025automatic, zheng2026navigating}, and combining similarity-based sampling with comparison-reward-driven reinforcement learning \cite{zheng2026isolated}.

Despite incorporating broader domain information, most methods still overlook recent, non-contemporaneous knowledge in the field. Such knowledge reflects patterns of scholarly communication and inheritance and is crucial for assessing the true contribution of a new paper. Although integrating non-contemporaneous knowledge has been explored in graph-based impact prediction methods \cite{he2023h2cgl, xue2024predicting}, it remains underexplored in LLM-based systems. While \citet{zhu2025deepreview} includes a retrieval module in its workflow, it does not empirically validate its effectiveness. Only \citet{zhao2025words} makes substantial use of this information by retrieving related papers via external tools and predicting standardized impact scores through listwise ranking.

Overall, existing approaches almost always study these information sources in isolation and rarely integrate them effectively. A promising and important direction is to consider them jointly within a unified graph-based system.

\subsection{Graph Methods for LLMs}

Before the rise of large language models, the dominant paradigm for graph tasks was message-passing-based graph neural networks (GNNs), which perform representation learning and downstream prediction via neighborhood aggregation \cite{kipf2016semi,velivckovic2017graph,gilmer2017neural,hamilton2017inductive}.

As LLMs have advanced in language understanding and knowledge representation, graph methods have gradually become complementary to LLMs. Some early approaches rewrote local graph structures into natural language descriptions using handcrafted rules and fed them into models as prompts \cite{wang2023can, zhao2023graphtext}.

Recent studies aim to align graph structure and linguistic semantics at the levels of representation space and reasoning mechanisms. Through instruction tuning or adapter modules, they endow LLMs with structured graph understanding. For example, GraphGPT \cite{tang2024graphgpt} integrates LLMs with graph structural knowledge via graph instruction tuning; UniGTE \cite{wang2025unigte} unifies structure and semantics by attending over graph tokens and natural language, enabling cross-task reasoning; \citet{jing2026entropy} achieves efficient alignment between a frozen graph encoder and a frozen LLM; G-reasoner \cite{luo2025g} unifies graph and language foundation models into a general graph representation for scalable reasoning; \citet{tao2025code} incorporates repository code graphs into attention and maps node attributes through adapters; and \citet{feng2025grapheval} use a lightweight graph-based LLM framework for idea evaluation.

Meanwhile, another line of work leverages graphs for retrieval-augmented generation, using knowledge graphs as external knowledge bases. Representative examples include GraphRAG \cite{edge2024local}, which introduces a community summarization and stepwise aggregation framework; LightRAG \cite{guo2024lightrag}, which combines graph structure with vector representations via a two-level retrieval scheme; HippoRAG \cite{gutierrez2024hipporag}, which performs memory-style retrieval with personalized PageRank; and an integrated system that adopts agent-based ideas to unify retrieval, graph construction, and community detection \cite{dong2025youtu}.

Other approaches treat graphs as interactive environments, mitigating LLM limitations in multi-step logical reasoning, precise computation, and spatiotemporal awareness through multi-step planning and tool use \cite{zhang2023graph}, including generating code to execute graph algorithms \cite{finkelshtein2025actions}.

However, for semantically centered paper review problems, graph methods that primarily rely on modality alignment or external retrieval are not directly reusable. The texts in our setting exhibit fine-grained semantic and argumentative relations that are tightly coupled with the review task. Forcing text into node embeddings, or making results explicit via alignment and averaging, can lead to semantic compression and partially undermine the transparency and interpretability of review texts. We therefore turn to a redesign of the overall graph-based approach.

\section{Method Details}

\subsection{Sequential 2-Factor Matching}
\label{sec:sequential_2_factor_matching}

\begin{algorithm}[t]
	\caption{Sequential 2-Factor Matching for Evidence Expansion}
	\label{alg:2factor}
	\begin{algorithmic}[1]
		\small
		\REQUIRE Node embeddings $\{\mathbf{h}_u\}_{u=1}^{N}$, graph-expansion budget $T$
		\ENSURE Adjacency matrix $\mathbf{C}^{(T)}$
		\STATE \textbf{Initialize} $t\gets 0$, $\mathbf{C}^{(0)}\gets \mathbf{0}$
		\WHILE{$t < T$}
		\STATE Obtain $\Delta\mathbf{C}^{(t)}$ by solving:
		\setlength{\jot}{2pt}
		\setlength{\abovedisplayskip}{0pt}
		\setlength{\belowdisplayskip}{-10pt}
		\begin{align*}
			\max_{\Delta c^{(t)}_{uv}}\quad 
			& \sum_{u=1}^n \sum_{v=1}^n \mathbf{h}_u^\top \mathbf{h}_v \cdot \Delta c^{(t)}_{uv} \qquad\qquad\quad\\
			\text{s.t.}\quad 
			& \sum_{v=1}^n \Delta c^{(t)}_{uv} = 2,\\
			& \Delta c^{(t)}_{uv} = \Delta c^{(t)}_{vu},\\
			& \Delta c^{(t)}_{uv} \in \{0,1\},\\
			& \Delta c^{(t)}_{uv} \le 1 - c^{(t-1)}_{uv}.
		\end{align*}
		\STATE $\mathbf{C}^{(t)} \gets \mathbf{C}^{(t-1)} + \Delta \mathbf{C}^{(t)}$
		\STATE $t \gets t+1$
		\ENDWHILE
		\RETURN $\mathbf{C}^{(T)}$
	\end{algorithmic}
\end{algorithm}

We propose a Sequential 2-Factor Matching algorithm to efficiently construct a sparse comparison adjacency matrix with time complexity $O(TN)$, where $T \ll N$. The full procedure is given in Algorithm~\ref{alg:2factor}.
Let $\mathbf{C}^{(t)} \in \mathbb{R}^{N \times N}$ denote the matching state after expansion step $t$, where $c_{uv}^{(t)}$ indicates whether an edge has been activated between nodes $u$ and $v$. At each step, we solve a maximum-weight 2-factor matching problem to obtain $\Delta \mathbf{C}^{(t)}$, which adds high-similarity edges based on pairwise inner products $\mathbf{h}_u^\top \mathbf{h}_v$. The constraints enforce binary and symmetric edge assignments, require each node to connect to exactly two new neighbors in the current step, and exclude previously selected edges. We then update the graph by $\mathbf{C}^{(t)} = \mathbf{C}^{(t-1)} + \Delta \mathbf{C}^{(t)}$ and repeat the process for the fixed budget $T$. Thus, $T$ controls comparison coverage, and the final degree of each node is $2T$.

This matching procedure has several desirable properties:
\begin{enumerate}
	\item $\mathbf{C}^{(T)}$ is the sum of $T$ disjoint 2-factors and satisfies greedy subset inclusion:
	$\mathbf{C}^{(1)} \subset \mathbf{C}^{(2)} \subset \cdots \subset \mathbf{C}^{(T)}$,
	which ensures consistent behavior under scaling. The final degree of each node is $D=2T$.
	\item Assuming a unique optimal 2-factor exists, the algorithm is permutation equivariant, meaning the optimal result is invariant to node relabeling.
	\item For $T\ge 1$ and $N\ge 3$, the loop executes at least once, so the graph always contains edges. This follows from the fact that any complete graph $K_N$ with $N\ge 3$ admits a 2-factor.
\end{enumerate}

\subsection{Text Consolidation}
\label{sec:text_consolidation}

In practice, for node-level inference, an LLM's output can be decomposed into two types of information, denoted as \(f_{\mathrm{LLM}}(x_v, p_{\mathrm{s}})=({\hat y}_{\mathrm{s}}, {\hat t}_{\mathrm{s}})\). Here, \({\hat y}_{\mathrm{s}}\) is a multi-class probabilistic label distribution aligned with the training objective, which can be further converted into a scalar score, and \({\hat t}_{\mathrm{s}}\) is a natural-language rationale for the review. Similarly, for edge-level inference, the model output is \(f_{\mathrm{LLM}}(x_u, x_v, p_{\mathrm{c}})=({\hat y}_{\mathrm{c}}, {\hat t}_{\mathrm{c}})\), where \({\hat y}_{\mathrm{c}}\) represents a preference signal such as \(x_u \succ x_v\) and is used directly for score construction, while \({\hat t}_{\mathrm{c}}\) provides the corresponding natural-language rationale.

PPR operates on numerical priors and directed preference transitions rather than textual rationales. We therefore introduce a parallel text-processing procedure alongside PPR-based ranking aggregation. Specifically, PPR numerically aggregates the pointwise prior and directed preferences over the fixed graph, whereas the rationale texts associated with nodes and edges are collected, organized, and merged in parallel.

Following common review-writing conventions, we further reformat these scattered text fragments into a coherent evaluation paragraph. We design a lightweight prompt that instructs a model to reformat the collected texts. Under this prompt, the model consistently restructures information from multiple sources and produces a consolidated evaluation with improved organization and coherence. We implement this step via an external instruction-following model as a text refinement module, instantiated as DeepSeek-V3.2 in our experiments. This module is model-agnostic and can be replaced by any instruction-following LLM with basic text organization capabilities.

The refined evaluation text serves two purposes. First, it presents the semantic content in a clearer and more readable form for human authors and reviewers. Second, it enables qualitative comparison with texts generated by other methods in our experiments.

\subsection{Iterative Prompt Evolving}
\label{sec:iterative_prompt_evolving}

Paper reviewing is a specialized and cognitively demanding task. Even domain experts may find it challenging to articulate evaluation criteria that are both explicit and sufficiently fine-grained. To obtain prompts that are robust and comprehensive enough to cover the multiple dimensions of paper assessment, we develop a task-specific evolving procedure for prompt construction. As shown in Figure \ref{fig:train_process} (right), this process can be written as $p^{(0)} \rightarrow p^{(1)} \rightarrow \cdots \rightarrow p^{(M)}$. Starting from an initial prompt $p^{(0)}$, the teacher model iteratively refines the prompt over $M$ rounds. At each step, the model compares the outputs produced before and after refinement, and retains the prompt that leads to the better result. The final prompt, denoted by $p^* = p^{(M)}$, is then used to generate cold-start responses. These responses further serve as SFT data that teach the student model the evaluation dimensions and criteria underlying paper review decisions.

The update at iteration $m$ is defined as: 
\begin{equation}
	p^{(m)} = f_{\mathrm{Judger}}\big(p^{(m-1)}, f_{\mathrm{Evolver}}(p^{(m-1)})\big)
\end{equation}

At iteration $m$, the procedure takes the current best prompt from the previous round, $p^{(m-1)}$, as input and produces a refined candidate through $f_{\mathrm{Evolver}}$. The function $f_{\mathrm{Judger}}$ then compares the current prompt with its refined version and selects the one that yields the better outcome. If the refined prompt improves performance, it is accepted as new $p^{(m)}$; otherwise, the previous prompt is preserved. Both $f_{\mathrm{Evolver}}$ and $f_{\mathrm{Judger}}$ are implemented with a designated teacher model through prompt-based optimization. In our implementation, we use DeepSeek-V3.2 as the teacher model by default.

\subsection{Reward-Induced Maximum Likelihood}
\label{sec:reward_induced_maximum_likelihood}

RIML is used in training for both node and edge capabilities. At the node level, the target distribution is softly induced by distance-based rewards over score anchors. At the edge level, the target distribution is a binary one-hot label induced by the pairwise ranking relation weighted by the score difference. For the complete training pipeline, see Figure \ref{fig:train_process} (right).

\paragraph{Node-level training}
For each sample $v$, let $s_v$ denote its ground-truth score and let the input query be $q_v=(x_v,p_{\mathrm{s}})$, where $x_v$ is the paper content and $p_{\mathrm{s}}$ is the prompt for the single-paper scoring task.

Because review scores in many top conferences are not restricted to integer values, we introduce a set of score anchors to approximate the continuous score space using a finite set of proxy classes:
\begin{equation}
	\mathbf{a}=\{a_1,a_2,\dots,a_K\}
\end{equation}

The model outputs logits over these proxy class tokens:
\begin{equation}
	\mathbf{z}_v=(z_{v1},z_{v2},\dots,z_{vK})\in\mathbb{R}^K
\end{equation}

The predictive distribution is defined as:
\begin{equation}
	P(k\mid q_v)=\frac{\exp(z_{vk})}{\sum_{l=1}^{K}\exp(z_{vl})}
\end{equation}

Rather than hard-assigning $s_v$ to its nearest anchor, we define a distance-based reward for each anchor $a_k$:
\begin{equation}
	r_{vk}=-\frac{(s_v-a_k)^2}{2\sigma^2}
\end{equation}

Here, $\sigma$ controls the smoothness of the supervision distribution. The reward-induced target distribution is then given by:
\begin{equation}
	y_{vk}=\frac{\exp(r_{vk})}{\sum_{l=1}^{K}\exp(r_{vl})}
\end{equation}

The node-level objective is defined as:
\begin{equation}
	\mathcal{L}_{\mathrm{node}}(\theta)
	=
	-\frac{1}{N}\sum_{v=1}^{N}\sum_{k=1}^{K} w_v y_{vk}\log P(k\mid q_v)
\end{equation}

We set $w_v=1$ for all samples. This soft supervision preserves relative positional information in the continuous score space through anchor-based distance modeling.

\paragraph{Edge-level training}
For each sample pair $(u,v)$, let $s_u$ and $s_v$ denote their ground-truth scores. We require $|s_u-s_v|>\delta$ to ensure sufficiently reliable supervision. We further construct training pairs with a greedy matching strategy such that each sample appears in at most one pair per training epoch, which improves pair diversity. The input query is $q_{uv}=(x_u,x_v,p_{\mathrm{c}})$, where $x_u$ and $x_v$ are the contents of the two papers and $p_{\mathrm{c}}$ is the prompt for the pairwise comparison task.

We formulate pairwise comparison as a binary classification problem. The label indicating whether sample $u$ receives a higher score than sample $v$ is:
\begin{equation}
	k_{uv}=\mathbb{I}[s_u-s_v>0]
\end{equation}

The model outputs logits over two proxy class tokens:
\begin{equation}
	\mathbf{z}_{uv}=(z_{uv0},z_{uv1})\in\mathbb{R}^2
\end{equation}

The predictive distribution is defined as:
\begin{equation}
	P(k\mid q_{uv})=\frac{\exp(z_{uvk})}{\exp(z_{uv0})+\exp(z_{uv1})}
\end{equation}

Unlike node-level training, edge-level training uses a one-hot target distribution induced by the ground-truth comparison relation:
\begin{equation}
	r_{uvk}=\mathbb{I}[k=k_{uv}],  \quad y_{uvk}=r_{uvk}
\end{equation}

The edge-level objective is defined as:
\begin{equation}
	\mathcal{L}_{\mathrm{edge}}(\theta)
	=
	-\frac{1}{M}\sum_{m=1}^{M}\sum_{k=0}^{1} w_{uv} y_{uvk}\log P(k\mid q_{uv})
\end{equation}

Here, $M$ is the upper bound on the total number of valid edges, and each pair $(u,v)$ corresponds to one edge index $m$. We define the edge weight as $w_{uv}=|s_u-s_v|$. A larger score gap indicates a clearer ranking relation and thus provides a more reliable supervision signal.

\begin{figure*}[htbp]
	\setlength{\belowcaptionskip}{-5pt} 
	\setlength{\floatsep}{10pt} 
	\setlength{\textfloatsep}{10pt}
	\centering
	\includegraphics[width=\textwidth]{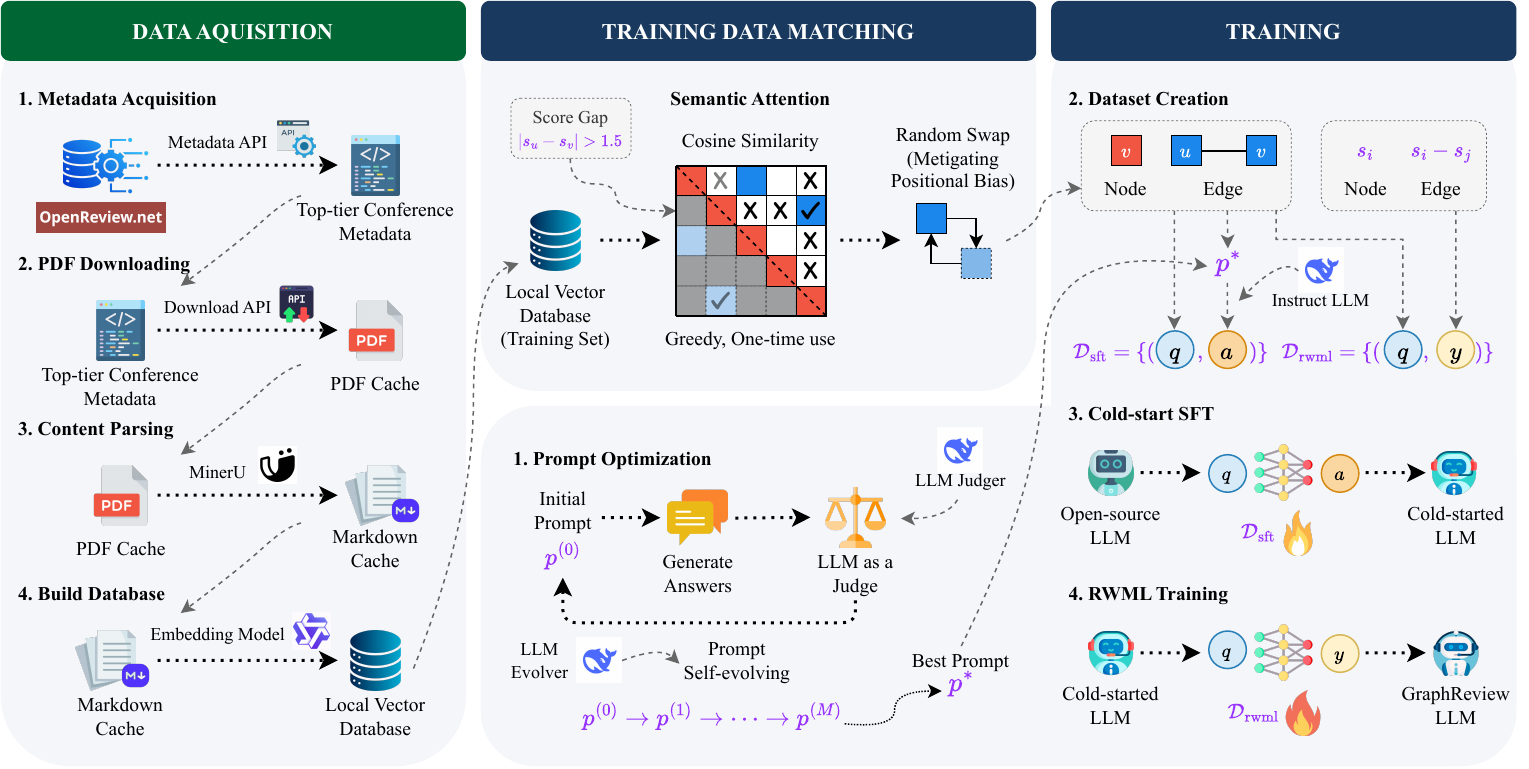}
	\caption{Pipeline of dataset construction (left) and training process (right).}
	\label{fig:train_process}
\end{figure*}

\section{Experiment Details}

\subsection{Dataset Construction}
\label{sec:dataset_construction}

As shown in Figure \ref{fig:train_process} (left), we first collect the full text of all papers through the OpenReview API, parse the PDFs with MinerU, an open-source OCR tool, and convert them into Markdown text. We then use the pre-trained semantic encoder Qwen3-8B-Embedding to generate paper embeddings and store them in a local vector database for subsequent graph construction.

We construct the primary supervised dataset using the full set of ICLR 2025 submissions, following the same split protocol as DeepReview. Specifically, the average human reviewer score serves as the ground-truth label for paper quality. The ICLR 2025 data are split into training, validation, and test sets, with approximately 8K papers for training and validation combined and 0.5K papers for testing. The training and validation sets are further divided at a ratio of 9:1.
To build the retrieval-augmented external knowledge base, we additionally collect all accepted papers from ICLR, ICML, and NeurIPS in 2023 and 2024, yielding a total of about 15K papers. This auxiliary corpus is designed to capture recent research trends and provide up-to-date background knowledge for retrieval and graph construction.

To further evaluate cross-year transferability and generalization, we additionally sample 500 papers each from ICLR 2026 and ICML 2025 for out-of-domain evaluation, while using the same external knowledge base. These evaluation sets include expert-provided paper group annotations and cover multiple quality levels.

\subsection{Training Setup and Hyperparameters}
\label{sec:training_setup_and_hyperparameters}

The model is trained in two stages. In the SFT stage, the learning rate is fixed to $2\times 10^{-4}$, the batch size is 1, and gradient accumulation spans 4 steps. In the RIML stage, we use AdamW with learning rates of $1\times 10^{-4}$, a batch size of 2, gradient accumulation over 16 steps, and a cosine learning rate scheduler. All training and inference experiments are conducted on a single RTX Pro 6000 96G GPU. The total training time is approximately 10 hours. During inference, we use vLLM for text generation. During inference-time graph evidence expansion, we set $n$, the total number of test papers, to 500. Since incorporating the entire external knowledge base would incur substantial computational cost, we set $N$ to 1000 for efficiency and randomly sample these instances across different years to validate the effectiveness of the model. The minimum improvement threshold $\epsilon$ is set to 0.01, and the maximum patience $c_{\max}$ is set to 3. The damping factor $\lambda$ in PPR is consistently set to 0.2, in accordance with the best result from the hyperparameter study. For the anchors vector $\mathbf{a}$ in node-level training, we adopt the ICLR 2025 rating scale levels $\{1, 3, 5, 6, 8, 10\}$. For the score difference threshold $\delta$ in edge-level training, the value is set to 1.5. All random seeds set to 42.

\subsection{Baselines}
\label{sec:baselines}

We reproduce a diverse set of baselines to benchmark our method from multiple perspectives. To ensure a fair evaluation, we explicitly account for potential data leakage issues, enabling a systematic comparison with both traditional and contemporary approaches.

\paragraph{Naive LLMs}

For this family of models, we directly query general-purpose LLMs via the APIs for score prediction. All samples use the same scoring prompt template, which is obtained through our proposed prompt evolution mechanism, consistent with the prompt template implemented in our scoring model. During inference, the temperature is fixed at 0, and all other parameters remain at their default values. We do not use tool calling, external retrieval, or any other auxiliary components. The model produces the final prediction in a single inference round.

\paragraph{Classic GNNs}

We implement a set of GNN baselines on the paper graph, where each node represents a submission and node features are derived from the paper text. For each paper, we concatenate the title and abstract, and encode the resulting text with Qwen3-Embedding-8B.

We build a semantic kNN graph using cosine distance in the embedding space. For each node, we connect it to its top-$k$ nearest neighbors according to cosine distance computed from the pretrained text embeddings. In all settings, we remove duplicate edges when needed and add self-loops to every node. We set $k=10$. We formulate the task as node-level regression and train all GNNs to predict the scalar target score using MSE. All experiments use a unified DGL implementation for graph construction. 

We evaluate three standard two-layer architectures: GCN, GAT, and GraphSAGE. The GCN baseline contains two graph convolution layers, with ReLU activation and dropout applied after the first layer. The GAT baseline uses an 8-head attention layer followed by a single-head output layer. The first layer uses ELU activation, and both attention dropout and feature dropout follow standard practice. The GraphSAGE baseline adopts the mean aggregator in both layers, with ReLU and dropout applied between layers. The hidden dimension is set to 128 for all models. For GAT, this dimension is evenly split across attention heads. We optimize all models with Adam using a learning rate of $1 \times 10^{-3}$, train for at most 200 epochs, and apply early stopping based on validation loss with a patience of 5 epochs. The dropout rate is 0.5 for GCN and GraphSAGE, and 0.6 for GAT, reflecting the common observation that attention-based GNNs often benefit from stronger regularization. All hyperparameters are selected through careful multi-round tuning.

\paragraph{Review Agents and Comparative Systems}

We compare multiple categories of LLM-based reviewing systems, including general agentic reviewing frameworks, specialized reviewing models released by their authors, and comparative methods that require retraining under our data split. To ensure fair comparison and reliable conclusions, we standardize the backbone model, input information, and data split whenever possible, and assess potential data leakage on a case-by-case basis.

For AIScientist, AgentReview, and PairReview, we use GPT-oss-120B as the unified backbone model. In our experiments, these systems take the full paper as input and output the final score prediction. For SEA-7B, CycleReviewer-7B, and NAIPv1-8B, we directly evaluate the open-source models released by their authors. For DeepReview-7B, DeepReview-14B, and CNPE-7B, the original authors trained these models on the ICLR 2025 dataset, and their train-test split is broadly aligned with ours. There is no overlap between the test samples used in this paper and the training data of these models. Therefore, evaluating them on our test set does not introduce data leakage. For NAIPv2-8B, the original ICLR 2025 split used in its reported experiments differs from ours. Consequently, directly using the released weights would make it difficult to ensure a fair comparison under a unified evaluation protocol. We therefore retrain this model on our split using the training code released by the authors, and report results on the same test set.

\begin{table*}[t]
    \renewcommand{\arraystretch}{0.82}
	\setlength{\belowcaptionskip}{-8pt}
	\setlength{\floatsep}{10pt}
	\setlength{\textfloatsep}{10pt}
	\centering
	\scriptsize
	\setlength{\tabcolsep}{5pt}
	\begin{tabularx}{\textwidth}{@{}
			>{\raggedright\arraybackslash}p{0.24\textwidth}
			*{6}{>{\centering\arraybackslash}X}@{}
			>{\centering\arraybackslash}p{0.14\textwidth}}
		\toprule
		& \multicolumn{3}{c}{\textbf{Decision Performance}} 
		& \multicolumn{3}{c}{\textbf{Ranking Performance}} 
		&  \\
		\cmidrule(lr){2-4} \cmidrule(lr){5-7}
		\textbf{Name} & \textbf{Accuracy} & \textbf{F1} & \textbf{AUC} & \textbf{Spearman $\rho$} & \textbf{Kendall $\tau$} & \textbf{NDCG@10} & \textbf{Avg. Performance} \\
		\midrule
		Gemini-3-Flash-Preview & \underline{0.6780} & \underline{0.6230} & \underline{0.7278} & \underline{0.5128} & \underline{0.4181} & 0.8136 & \underline{0.6289} \\
		GLM-5                  & 0.6540             & 0.5949             & 0.6737             & 0.4018             & 0.3274             & \underline{0.8144} & 0.5777             \\
		Qwen3.6-Plus           & 0.6660             & 0.6089             & 0.7174             & 0.4575             & 0.3708             & 0.6390             & 0.5766             \\
		DeepSeek-V4-Pro        & \underline{0.6780} & \underline{0.6230} & 0.7146             & 0.4046             & 0.3213             & 0.6829             & 0.5707             \\
		\textbf{GraphReview}  & \textbf{0.8980}    & \textbf{0.8806}    & \textbf{0.9590}    & \textbf{0.6626}    & \textbf{0.4808}    & \textbf{0.8547}    & \textbf{0.7893}    \\
		\bottomrule
	\end{tabularx}
	\caption{Performance of recently released LLMs. \textbf{Best} and \underline{second-best} results are highlighted.}
	\label{tab:recent_llm_performance}
\end{table*}

\subsection{Evaluation Metrics}
\label{sec:evaluation_metrics}

We evaluate model performance from two complementary perspectives: \textbf{binary decision quality} and \textbf{ranking quality}. The former assesses whether the predicted decision matches the ground-truth label, while the latter evaluates the quality of the predicted ordering, which is particularly important for downstream applications such as recommendation, retrieval, and literature intelligence.

For binary decision evaluation, we report \textbf{Accuracy}, \textbf{F1}, and \textbf{AUC}. 
For ranking evaluation, we report \textbf{Spearman's $\rho$}, \textbf{Kendall's $\tau$}, and \textbf{NDCG@10}. 

\paragraph{Accuracy} Accuracy is defined as:

\begin{equation}
    \mathrm{Accuracy}=\frac{1}{N}\sum_{i=1}^{N}\mathbb{I}[\hat{d}_i=d_i]
\end{equation}

Here, $d_i \in \{0,1\}$ and $\hat{d}_i \in \{0,1\}$ denote the ground-truth and predicted binary labels, respectively. 

\paragraph{F1} F1, namely Macro-F1, is computed by first calculating the F1 score for each class $c \in \mathcal{C}$:

\begin{equation}
    \mathrm{F1}_c = \frac{2\,\mathrm{Precision}_c\,\mathrm{Recall}_c}{\mathrm{Precision}_c+\mathrm{Recall}_c}
\end{equation}

And then averaging across classes: 

\begin{equation}
    \mathrm{F1} = \frac{1}{|\mathcal{C}|}\sum_{c\in\mathcal{C}} \mathrm{F1}_c
\end{equation}

\paragraph{AUC} AUC measures the probability that a randomly selected positive instance is assigned a higher predicted score than a randomly selected negative instance: 
\begin{equation}
    \begin{aligned}
        \mathrm{AUC} = & \frac{1}{|\mathcal{P}||\mathcal{N}|}\sum_{i\in\mathcal{P}}\sum_{j\in\mathcal{N}} \\
        & \left( \mathbb{I}[\hat{y}_i>\hat{y}_j] + \frac{1}{2}\mathbb{I}[\hat{y}_i=\hat{y}_j] \right)
    \end{aligned}
\end{equation}

Here, $\mathcal{P}$ and $\mathcal{N}$ denote the sets of positive and negative instances.

\paragraph{Spearman's $\rho$} Spearman's $\rho$ measures the correlation between the ground-truth and predicted rankings:

\begin{equation}
    \rho = 1 - \frac{6\sum_{i=1}^{N}\Delta_i^2}{N(N^2-1)} 
\end{equation}

Here, $\Delta_i$ denotes the difference between the rank of $y_i$ and that of $\hat{y}_i$. 

\paragraph{Kendall's $\tau$} Kendall's $\tau$ measures pairwise rank agreement with tie correction:

\begin{equation}
    \tau = \frac{n_c-n_d}{\sqrt{(n_0-n_x)(n_0-n_y)}} 
\end{equation}

Here, $n_c$ and $n_d$ denote the numbers of concordant and discordant pairs, $n_0 = N(N-1)/2$, and $n_x$ and $n_y$ denote the numbers of pairs tied only in the predicted and ground-truth rankings, respectively. 

\paragraph{NDCG@10} NDCG@10 evaluates the quality of the top-ranked results: 

\begin{equation}
    \mathrm{NDCG}@10 = \frac{\mathrm{DCG}@10}{\mathrm{IDCG}@10} 
\end{equation}

Specifically: 

\begin{equation}
    \mathrm{DCG}@10 = \sum_{i=1}^{10}\frac{2^{y_{(i)}}-1}{\log_2(i+1)}
\end{equation}

$y_{(i)}$ is the ground-truth relevance of the item ranked at position $i$, and $\mathrm{IDCG}@10$ is the maximum achievable DCG@10.

\section{Additional Results}
\label{sec:additional_results}

\subsection{Performance of Recent LLMs}
\label{sec:recent_llms}

In addition to the main LLM baselines, we also evaluate several recently released models. Their results are reported in Table~\ref{tab:recent_llm_performance}. Since all of these models were released after December 2025, they carry substantial leakage risk with respect to our benchmark. We therefore include them only as a reference for empirical comparison. Despite this favorable condition, their performance remains consistently below that of GraphReview. This result further indicates that stronger general-purpose LLMs alone do not close the gap in paper review, while our framework retains a clear advantage.

\providecommand{\scoreci}[3]{\shortstack{#1\\{[#2,#3]}}}

\subsection{Bootstrap Confidence Intervals}
\label{sec:bootstrap_confidence_intervals}

We report 95\% bootstrap confidence intervals for the main experimental results in Table~\ref{tab:main_bootstrap_ci}. The intervals are computed on the same test set as the main experimental table. GraphReview maintains a clear advantage over all baselines in decision, ranking, and average performance. The overall performance interval of GraphReview, is separated from that of the strongest baseline NAIPv2-8B, indicating that the improvement is robust under bootstrap resampling.

\subsection{Full Ablation Results}
\label{sec:full_ablation}

For completeness, we report the full ablation results on all evaluation metrics in Table~\ref{tab:full_ablation_results}. Although several alternatives are competitive on individual ranking metrics, the full model achieves the best overall performance, confirming that graph construction, graph propagation, and model training jointly contribute to the final system.

\FloatBarrier

\subsection{Difference Tests of Generalization}
\label{sec:difference_tests}

We also provide the detailed statistical test results for the generalization analysis in Table~\ref{tab:diff_test_results}. Since rankings do not naturally follow a normal distribution, non-parametric tests are used. For both conferences, the Kruskal-Wallis test shows significant overall differences across decision categories. The pairwise Mann-Whitney U tests are also significant for all category pairs. The result pattern aligns with practical review intuition: accepted and rejected papers usually differ substantially in quality, whereas accepted papers have already reached a relatively high standard, making within-group differences less pronounced. These results provide additional evidence that the predicted ranking scores align well with real review outcomes and remain discriminative across different decision levels.

\begin{table*}[t]
	\renewcommand{\arraystretch}{0.82}
	\setlength{\belowcaptionskip}{-8pt}
	\setlength{\floatsep}{10pt}
	\setlength{\textfloatsep}{10pt}
	\setlength{\aboverulesep}{1.5pt}
	\setlength{\belowrulesep}{1.5pt}
	\centering
	\scriptsize
	\setlength{\tabcolsep}{1.5pt}
	\begin{tabularx}{\textwidth}{@{}
			>{\raggedright\arraybackslash}p{0.16\textwidth}
			*{6}{>{\centering\arraybackslash}X}@{}
			>{\centering\arraybackslash}p{0.12\textwidth}}
		\toprule
		& \multicolumn{3}{c}{\textbf{Decision Performance}}
		& \multicolumn{3}{c}{\textbf{Ranking Performance}}
		& \\
		\cmidrule(lr){2-4} \cmidrule(lr){5-7}
		\textbf{Method} & \textbf{Accuracy} & \textbf{F1} & \textbf{AUC} & \textbf{Spearman $\rho$} & \textbf{Kendall $\tau$} & \textbf{NDCG@10} & \textbf{Avg. Performance} \\
		\specialrule{\lightrulewidth}{0.5pt}{0.5pt}
		\rowcolor{gray!10}
		\multicolumn{8}{@{}l}{\textbf{Classic GNNs}} \\
		GCN & \scoreci{0.5980}{0.5560}{0.6380} & \scoreci{0.5404}{0.4943}{0.5807} & \scoreci{0.5720}{0.5194}{0.6203} & \scoreci{0.1267}{0.0352}{0.2060} & \scoreci{0.0852}{0.0231}{0.1398} & \scoreci{0.5552}{0.4767}{0.6441} & \scoreci{0.4129}{0.3508}{0.4715} \\
		GAT & \scoreci{0.5760}{0.5340}{0.6200} & \scoreci{0.5235}{0.4798}{0.5684} & \scoreci{0.5104}{0.4562}{0.5627} & \scoreci{0.0485}{-0.0395}{0.1359} & \scoreci{0.0349}{-0.0255}{0.0953} & \scoreci{0.5734}{0.4595}{0.6728} & \scoreci{0.3778}{0.3108}{0.4425} \\
		GraphSAGE & \scoreci{0.5640}{0.5240}{0.6080} & \scoreci{0.5143}{0.4697}{0.5600} & \scoreci{0.5055}{0.4522}{0.5588} & \scoreci{0.0492}{-0.0364}{0.1420} & \scoreci{0.0348}{-0.0239}{0.1004} & \scoreci{0.5868}{0.4739}{0.7175} & \scoreci{0.3758}{0.3099}{0.4478} \\
		\specialrule{\lightrulewidth}{0.5pt}{0.5pt}
		\rowcolor{gray!10}
		\multicolumn{8}{@{}l}{\textbf{Naïve LLMs}} \\
		GPT-5-Mini & \scoreci{0.7100}{0.6740}{0.7520} & \scoreci{0.6605}{0.6186}{0.7055} & \scoreci{0.7104}{0.6661}{0.7537} & \scoreci{0.4136}{0.3406}{0.4889} & \underline{\scoreci{0.3317}{0.2724}{0.3954}} & \scoreci{0.6982}{0.5969}{0.8109} & \scoreci{0.5874}{0.5281}{0.6511} \\
		Gemini-2.5-Flash & \scoreci{0.6060}{0.5620}{0.6500} & \scoreci{0.5387}{0.4942}{0.5832} & \scoreci{0.5557}{0.5194}{0.5933} & \scoreci{0.1936}{0.0995}{0.2724} & \scoreci{0.1569}{0.0813}{0.2190} & \scoreci{0.6422}{0.5468}{0.7368} & \scoreci{0.4488}{0.3838}{0.5091} \\
		DeepSeek-V3.2 & \scoreci{0.6180}{0.5760}{0.6600} & \scoreci{0.5527}{0.5070}{0.5967} & \scoreci{0.5975}{0.5749}{0.6212} & \scoreci{0.3536}{0.2817}{0.4242} & \scoreci{0.2908}{0.2319}{0.3496} & \scoreci{0.6614}{0.5407}{0.7618} & \scoreci{0.5123}{0.4520}{0.5689} \\
		\specialrule{\lightrulewidth}{0.5pt}{0.5pt}
		\rowcolor{gray!10}
		\multicolumn{8}{@{}l}{\textbf{Review Agents}} \\
		AgentReview & \scoreci{0.5160}{0.4760}{0.5601} & \scoreci{0.5074}{0.4666}{0.5511} & \scoreci{0.5528}{0.5031}{0.6050} & \scoreci{0.0042}{-0.0759}{0.0889} & \scoreci{0.0026}{-0.0591}{0.0673} & \scoreci{0.6411}{0.5133}{0.7722} & \scoreci{0.3707}{0.3040}{0.4408} \\
		AIScientist & \scoreci{0.7100}{0.6700}{0.7500} & \scoreci{0.5417}{0.4950}{0.5890} & \scoreci{0.6418}{0.5936}{0.6885} & \scoreci{0.3162}{0.2416}{0.3973} & \scoreci{0.2508}{0.1919}{0.3168} & \scoreci{0.7493}{0.5853}{0.8507} & \scoreci{0.5350}{0.4629}{0.5987} \\
		SEA-7B & \scoreci{0.6960}{0.6560}{0.7360} & \scoreci{0.4104}{0.3961}{0.4240} & \scoreci{0.5459}{0.4934}{0.5951} & \scoreci{0.0983}{0.0108}{0.1878} & \scoreci{0.0754}{0.0083}{0.1449} & \scoreci{0.6585}{0.5317}{0.7196} & \scoreci{0.4141}{0.3494}{0.4679} \\
		CycleReviewer-7B & \scoreci{0.6620}{0.6200}{0.7040} & \scoreci{0.5300}{0.4835}{0.5748} & \scoreci{0.6766}{0.6244}{0.7259} & \scoreci{0.3132}{0.2362}{0.3855} & \scoreci{0.2255}{0.1696}{0.2808} & \scoreci{0.6859}{0.5856}{0.7795} & \scoreci{0.5155}{0.4532}{0.5751} \\
		DeepReview-7B & \scoreci{0.6540}{0.6100}{0.6960} & \scoreci{0.5470}{0.5010}{0.5899} & \scoreci{0.5890}{0.5342}{0.6413} & \scoreci{0.2928}{0.2052}{0.3736} & \scoreci{0.2110}{0.1477}{0.2719} & \scoreci{0.6938}{0.6360}{0.7880} & \scoreci{0.4979}{0.4390}{0.5601} \\
		DeepReview-14B & \scoreci{0.6860}{0.6440}{0.7240} & \scoreci{0.6240}{0.5816}{0.6657} & \scoreci{0.6494}{0.6009}{0.6959} & \scoreci{0.3995}{0.3172}{0.4772} & \scoreci{0.2991}{0.2362}{0.3610} & \scoreci{0.6657}{0.5903}{0.7578} & \scoreci{0.5540}{0.4950}{0.6136} \\
		\specialrule{\lightrulewidth}{0.5pt}{0.5pt}
		\rowcolor{gray!10}
		\multicolumn{8}{@{}l}{\textbf{Comparative Systems}} \\
		PairReview & \scoreci{0.6700}{0.6300}{0.7120} & \scoreci{0.5701}{0.5267}{0.6208} & \scoreci{0.6047}{0.5526}{0.6587} & \scoreci{0.2585}{0.1750}{0.3428} & \scoreci{0.1781}{0.1201}{0.2383} & \scoreci{0.6644}{0.5756}{0.7442} & \scoreci{0.4910}{0.4300}{0.5528} \\
		CNPE-7B & \scoreci{0.7200}{0.6800}{0.7580} & \scoreci{0.6692}{0.6229}{0.7091} & \scoreci{0.7363}{0.6855}{0.7825} & \scoreci{0.3995}{0.3172}{0.4763} & \scoreci{0.2774}{0.2197}{0.3344} & \underline{\scoreci{0.8040}{0.7133}{0.8991}} & \scoreci{0.6010}{0.5398}{0.6599} \\
		NAIP-8B & \scoreci{0.6140}{0.5740}{0.6560} & \scoreci{0.5413}{0.4961}{0.5842} & \scoreci{0.5641}{0.5146}{0.6189} & \scoreci{0.1585}{0.0757}{0.2469} & \scoreci{0.1074}{0.0498}{0.1700} & \scoreci{0.6882}{0.5896}{0.7812} & \scoreci{0.4456}{0.3833}{0.5095} \\
		NAIPv2-8B & \underline{\scoreci{0.7260}{0.6880}{0.7640}} & \underline{\scoreci{0.6780}{0.6336}{0.7211}} & \underline{\scoreci{0.7627}{0.7144}{0.8044}} & \underline{\scoreci{0.4205}{0.3443}{0.4941}} & \scoreci{0.2928}{0.2392}{0.3480} & \scoreci{0.7723}{0.6975}{0.8720} & \underline{\scoreci{0.6087}{0.5528}{0.6673}} \\
		\specialrule{\lightrulewidth}{0.5pt}{0.5pt}
		\rowcolor{gray!10}
		\multicolumn{8}{@{}l}{\textbf{Ours}} \\
		GraphReview & \textbf{\scoreci{0.8980}{0.8720}{0.9240}}$^{*}$ & \textbf{\scoreci{0.8806}{0.8514}{0.9108}}$^{*}$ & \textbf{\scoreci{0.9590}{0.9424}{0.9747}}$^{*}$ & \textbf{\scoreci{0.6626}{0.6092}{0.7108}}$^{*}$ & \textbf{\scoreci{0.4808}{0.4351}{0.5237}}$^{*}$ & 
		\textbf{\scoreci{0.8547}{0.7756}{0.9397}} & 
		\textbf{\scoreci{0.7893}{0.7476}{0.8306}}$^{*}$ \\
		\bottomrule
	\end{tabularx}
	\caption{Main results with 95\% bootstrap confidence intervals. $^{*}$ denotes that GraphReview is significantly better than the second-best method under a paired bootstrap test with $p < 0.001$. \textbf{Best} and \underline{second-best} point estimates are highlighted.}
	\label{tab:main_bootstrap_ci}
\end{table*}

\begin{table*}[t]
	\renewcommand{\arraystretch}{0.82}
	\setlength{\belowcaptionskip}{-8pt}
	\setlength{\floatsep}{10pt}
	\setlength{\textfloatsep}{10pt}
	\setlength{\aboverulesep}{1.5pt}
	\setlength{\belowrulesep}{1.5pt}
	\centering
	\scriptsize
	\setlength{\tabcolsep}{5pt}
	\begin{tabularx}{\textwidth}{@{}
			>{\raggedright\arraybackslash}p{0.20\textwidth}
			*{6}{>{\centering\arraybackslash}X}@{}
			>{\centering\arraybackslash}p{0.14\textwidth}}
		\toprule
		& \multicolumn{3}{c}{\textbf{Decision Performance}}
		& \multicolumn{3}{c}{\textbf{Ranking Performance}}
		& \\
		\cmidrule(lr){2-4} \cmidrule(lr){5-7}
		\textbf{Variant} & \textbf{Accuracy} & \textbf{F1} & \textbf{AUC} & \textbf{Spearman $\rho$} & \textbf{Kendall $\tau$} & \textbf{NDCG@10} & \textbf{Avg. Performance} \\
		\specialrule{\lightrulewidth}{0.5pt}{0.5pt}
		\rowcolor{gray!10}
		\multicolumn{8}{@{}l}{\textbf{Information Sources}} \\
		\textcolor{red}{w/o } Graph & 0.7580 & 0.7167 & 0.8174 & 0.5529 & 0.3948 & 0.8032 & 0.6738 \\
		\textcolor{red}{w/o } Intrinsic Quality & 0.8820 & 0.8618 & 0.9497 & 0.5903 & 0.4243 & 0.8236 & 0.7553 \\
		\textcolor{red}{w/o } Synchronic Links & 0.8820 & 0.8618 & 0.9315 & \underline{0.6632} & \textbf{0.4823} & \underline{0.8514} & \underline{0.7787} \\
		\textcolor{red}{w/o } Diachronic Links & 0.8820 & 0.8618 & 0.9498 & 0.6541 & 0.4730 & 0.7745 & 0.7659 \\
		\specialrule{\lightrulewidth}{0.5pt}{0.5pt}
		\rowcolor{gray!10}
		\multicolumn{8}{@{}l}{\textbf{Training Strategies}} \\
		\textcolor{red}{w/o } RWML & 0.8500 & 0.8244 & 0.9068 & 0.5857 & 0.4187 & 0.7665 & 0.7253 \\
		\textcolor{red}{w/o } SFT & 0.8500 & 0.8244 & 0.9073 & 0.6036 & 0.4331 & 0.7835 & 0.7337 \\
		\specialrule{\lightrulewidth}{0.5pt}{0.5pt}
		\rowcolor{gray!10}
		\multicolumn{8}{@{}l}{\textbf{Graph Construction}} \\
		Random Connection & 0.8620 & 0.8384 & 0.9189 & 0.6447 & 0.4642 & 0.8077 & 0.7560 \\
		\specialrule{\lightrulewidth}{0.5pt}{0.5pt}
		\rowcolor{gray!10}
		\multicolumn{8}{@{}l}{\textbf{Aggregation Alternatives}} \\
		Naïve Win-Rate Averaging & 0.8820 & 0.8618 & \underline{0.9503} & 0.6422 & 0.4619 & 0.8342 & 0.7721 \\
		Bradley-Terry & \underline{0.8900} & \underline{0.8712} & 0.9478 & \textbf{0.6652} & 0.4785 & 0.8082 & 0.7768 \\
		Borda Count & 0.8460 & 0.8197 & 0.9097 & 0.6298 & 0.4527 & 0.8246 & 0.7471 \\
		\specialrule{\lightrulewidth}{0.5pt}{0.5pt}
		\rowcolor{gray!10}
		\multicolumn{8}{@{}l}{\textbf{Full Model}} \\
		GraphReview & \textbf{0.8980} & \textbf{0.8806} & \textbf{0.9590} & 0.6626 & \underline{0.4808} & \textbf{0.8547} & \textbf{0.7893} \\
		\bottomrule
	\end{tabularx}
	\caption{Full ablation results with detailed decision and ranking metrics. \textbf{Best} and \underline{second-best} results in each metric are highlighted.}
	\label{tab:full_ablation_results}
\end{table*}

\begin{table*}[t]
	\renewcommand{\arraystretch}{0.82}
	\setlength{\belowcaptionskip}{-8pt}
	\setlength{\floatsep}{10pt}
	\setlength{\textfloatsep}{10pt}
	\centering
	\scriptsize
	\setlength{\tabcolsep}{4pt}
	\begin{tabularx}{\textwidth}{@{}
			>{\raggedright\arraybackslash}p{0.12\textwidth}
			>{\centering\arraybackslash}p{0.08\textwidth}
			>{\centering\arraybackslash}p{0.08\textwidth}
			>{\centering\arraybackslash}X
			>{\centering\arraybackslash}X
			>{\centering\arraybackslash}X
			>{\centering\arraybackslash}X
			>{\centering\arraybackslash}X
			>{\centering\arraybackslash}X
			@{}}
		\toprule
		& \multicolumn{2}{c}{\textbf{Kruskal-Wallis Test}}
		& \multicolumn{6}{c}{\textbf{Pairwise Mann-Whitney U Test}} \\
		\cmidrule(lr){2-3} \cmidrule(lr){4-9}
		\textbf{Conference}
		& \multicolumn{2}{c}{}
		& \multicolumn{2}{c}{Reject \textcolor{gray}{vs} Poster}
		& \multicolumn{2}{c}{Reject \textcolor{gray}{vs} Spotlight/Oral}
		& \multicolumn{2}{c}{Poster \textcolor{gray}{vs} Spotlight/Oral} \\
		\cmidrule(lr){4-5} \cmidrule(lr){6-7} \cmidrule(lr){8-9}
		& $H$\textbf{-statistic} & $p$\textbf{-value}
		& $U$\textbf{-statistic} & $p$\textbf{-value}
		& $U$\textbf{-statistic} & $p$\textbf{-value}
		& $U$\textbf{-statistic} & $p$\textbf{-value} \\
		\midrule
		ICLR 2026
		& 273.0119
		& 0.000$^{*}$
		& 2086.0 & 0.000$^{*}$
		& 962.0 & 0.000$^{*}$
		& 10460.0 & 0.000$^{*}$ \\
		ICML 2025
		& 157.2580
		& 0.000$^{*}$
		& 7182.0 & 0.000$^{*}$
		& 1927.0 & 0.000$^{*}$
		& 9247.0 & 0.000$^{*}$ \\
		\bottomrule
	\end{tabularx}
	\caption{Results of Kruskal-Wallis and pairwise Mann-Whitney U tests across decision categories. $^{*}$ denotes statistical significance with $p < 0.001$.}
	\label{tab:diff_test_results}
\end{table*}

\section{Proof}
\label{sec:proof}
\newtheorem{theorem}{Theorem}

\begin{theorem}
	For every integer $N \ge 3$, the complete graph $K_N$ contains a 2-factor.
\end{theorem}

\begin{proof} 
	Let $V(K_N)=\{v_0,v_1,\ldots,v_{N-1}\}$. Consider the cycle $C=(v_0,v_1,v_2,\ldots,v_{N-1},v_0)$. Since $K_N$ is complete, it contains every edge of the form $\{v_i,v_{i+1 \bmod N}\}$ for $i\in\{0,1,\ldots,N-1\}$, so $C$ is a well-defined subgraph of $K_N$. Moreover, $C$ is spanning because it contains every node of $K_N$. For each node $v_i$, the two incident edges in $C$ are $\{v_i,v_{i-1 \bmod N}\}$ and $\{v_i,v_{i+1 \bmod N}\}$, and these two edges are distinct since $N\geq 3$. Hence every node has degree exactly $2$ in $C$. Therefore, $C$ is a $2$-factor of $K_N$.
\end{proof}

\section{Graph Evidence Expansion as Paradigm}

\begin{table*}[t]
    \renewcommand{\arraystretch}{0.82}
	\setlength{\belowcaptionskip}{-8pt}
	\setlength{\floatsep}{10pt}
	\setlength{\textfloatsep}{10pt}
	\centering
	\scriptsize
	\setlength{\tabcolsep}{5pt}
	\begin{tabularx}{\textwidth}{@{}
			>{\raggedright\arraybackslash}p{0.20\textwidth}
			*{6}{>{\centering\arraybackslash}X}@{}
			>{\centering\arraybackslash}p{0.14\textwidth}}
		\toprule
		& \multicolumn{3}{c}{\textbf{Decision Performance}} 
		& \multicolumn{3}{c}{\textbf{Ranking Performance}} 
		&  \\
		\cmidrule(lr){2-4} \cmidrule(lr){5-7}
		\textbf{Variant} & \textbf{Accuracy} & \textbf{F1} & \textbf{AUC} & \textbf{Spearman $\rho$} & \textbf{Kendall $\tau$} & \textbf{NDCG@10} & \textbf{Avg. Performance} \\
		\midrule
		CNPE-7B        & \underline{0.7180} & \underline{0.6698} & \underline{0.7354} & 0.3950 & 0.2753 & \underline{0.7752} & \underline{0.5948} \\
		DeepReview-14B & 0.6860             & 0.6240             & 0.6494             & \underline{0.3995} & \underline{0.2991} & 0.6657             & 0.5540             \\
		\textbf{Combined} & \textbf{0.7300} & \textbf{0.6839} & \textbf{0.7388} & \textbf{0.4604} & \textbf{0.3250} & \textbf{0.7838} & \textbf{0.6203} \\
		\bottomrule
	\end{tabularx}
	\caption{Performance of graph-based fusion with CNPE-7B and DeepReview-14B. \textbf{Best} and \underline{second-best} results are highlighted.}
	\label{tab:combined_cnpe_deepreview}
\end{table*}

\begin{table*}[t]
    \renewcommand{\arraystretch}{0.82}
	\setlength{\belowcaptionskip}{-8pt}
	\setlength{\floatsep}{10pt}
	\setlength{\textfloatsep}{10pt}
	\centering
	\scriptsize
	\setlength{\tabcolsep}{5pt}
	\begin{tabularx}{\textwidth}{@{}
			>{\raggedright\arraybackslash}p{0.20\textwidth}
			*{6}{>{\centering\arraybackslash}X}@{}
			>{\centering\arraybackslash}p{0.14\textwidth}}
		\toprule
		& \multicolumn{3}{c}{\textbf{Decision Performance}} 
		& \multicolumn{3}{c}{\textbf{Ranking Performance}} 
		&  \\
		\cmidrule(lr){2-4} \cmidrule(lr){5-7}
		\textbf{Variant} & \textbf{Accuracy} & \textbf{F1} & \textbf{AUC} & \textbf{Spearman $\rho$} & \textbf{Kendall $\tau$} & \textbf{NDCG@10} & \textbf{Avg. Performance} \\
		\midrule
		PairReview       & 0.6380             & \underline{0.5762} & 0.6080             & 0.2440             & 0.1670             & \underline{0.6843} & 0.4863             \\
		CycleReviewer-7B & \underline{0.6620} & 0.5300             & \underline{0.6766} & \underline{0.3132} & \underline{0.2255} & \textbf{0.6859}    & \underline{0.5155} \\
		\textbf{Combined} & \textbf{0.6780} & \textbf{0.6230} & \textbf{0.6794} & \textbf{0.3450} & \textbf{0.2359} & 0.6826 & \textbf{0.5407} \\
		\bottomrule
	\end{tabularx}
	\caption{Performance of graph-based fusion with PairReview and CycleReviewer-7B. \textbf{Best} and \underline{second-best} results are highlighted.}
	\label{tab:combined_pairreview_cyclereviewer}
\end{table*}

Graph-based LLM paper review can be developed into a complete and independent evidence-aggregation paradigm. GraphReview is not only a standalone method, but also a unified abstraction for a broader class of approaches that organize pointwise and pairwise review signals in a common graph structure. This abstraction is distinct from representation-learning GNNs: the graph is used as a symbolic structure for inference-time evidence expansion and global ranking aggregation. To support this view, we revisit mainstream methods through the lens of optimization objectives and show that their core ideas can be naturally unified within a graph-based framework. In particular, most existing LLM-based paper review methods can be characterized by two classic objectives.

\paragraph{Pointwise Objective} The first line of work emphasizes accurate scoring of individual papers and adopts a pointwise objective:
\begin{equation}
	\phi_{\mathrm{s}}^{*}
	=
	\arg\min_{\phi_{\mathrm{s}}}
	\sum_{x_v \in \mathbf{x}}
	\mathcal{L}\bigl(
	f_{\mathrm{s}}(x_v; \phi_{\mathrm{s}}),
	s_v
	\bigr)
\end{equation}

Here, $f_{\mathrm{s}}$ denotes a scoring function for estimating the quality of a single paper, with the objective of aligning the predicted score with the ground-truth label.

\paragraph{Pairwise Objective} The second line of work focuses on comparison and discrimination and adopts a pairwise objective:
\begin{equation}
	\phi_{\mathrm{c}}^{*}
	=
	\arg\min_{\phi_{\mathrm{c}}}
	\sum_{x_u, x_v \in \mathbf{x}}
	\mathcal{L}\bigl(
	f_{\mathrm{c}}(x_v, x_u; \phi_{\mathrm{c}}),
	s_u,
	s_v
	\bigr)
\end{equation}

Here, $f_{\mathrm{c}}$ denotes a comparison function over paper pairs, whose goal is to capture relative preference or comparative quality.

\paragraph{Combined Objective} From a graph perspective, pointwise and pairwise estimation are not separate formulations, but complementary signals that can be organized and globally integrated within a common evidence graph. We therefore write the combined optimization objective as:
\begin{equation}
	\begin{aligned}
		\Phi^{*}
		&=
		\arg\min_{\Phi}
		\mathcal{L} \\
		&\quad
		\Bigl(\bigl(
		f_{\mathrm{agg}}
		\circ
		\bigl(
		(f_{\mathrm{c}} \circ f_{\mathrm{att}})
		\oplus
		f_{\mathrm{s}}
		\bigr)
		\bigr)
		(\mathbf{x}; \Phi),
		\mathbf{r}^{\mathrm{gt}}
		\Bigr)
	\end{aligned}
\end{equation}

Here, $f_{\mathrm{att}}$ defines the graph construction strategy by selecting informative paper pairs, $f_{\mathrm{c}}$ and $f_{\mathrm{s}}$ generate edge-level and node-level signals, respectively, and $f_{\mathrm{agg}}$ aggregates these signals over the graph to produce the final prediction aligned with the ground-truth ranking $\mathbf{r}^{\mathrm{gt}}$.

Specifically, in GraphReview, $f_{\mathrm{att}}$ is instantiated as Sequential 2-Factor Matching, which selects informative comparison edges with low overhead. The aggregation function $f_{\mathrm{agg}}$ is instantiated as Personalized PageRank, which computes a stationary ranking over the fixed directed preference graph while incorporating the pointwise prior. The functions $f_{\mathrm{c}}$ and $f_{\mathrm{s}}$ correspond to the pairwise comparison model and the pointwise scoring model obtained from two-stage training.

\paragraph{Modular Decoupling} A key property of this framework is modular decoupling. GraphReview does not depend on any fixed backend model; rather, it lies in offering a graph-based formulation that can uniformly integrate heterogeneous review signals. In principle, any mainstream LLM or existing paper review method can serve as a backend. Its pointwise or pairwise predictions can be embedded into the graph and then combined through graph aggregation to produce a unified global ranking.

\paragraph{Experiments} To validate the effectiveness of this generalized graph evidence-expansion paradigm, we further replace the representative pair selection and comparison component $f_{\mathrm{c}} \circ f_{\mathrm{att}}$ and the direct scoring component $f_{\mathrm{s}}$ with several classic methods, so that they provide edge-level and node-level review signals, respectively. We keep Personalized PageRank as $f_{\mathrm{agg}}$.

We conduct ablation studies over different module combinations. Specifically, we evaluate two representative settings: (1) CNPE-7B \cite{zheng2026isolated} combined with DeepReview-14B \cite{zhu2025deepreview}, and (2) PairReview \cite{zhang2025from} combined with CycleReviewer-7B \cite{weng2025cycleresearcher}. 

The results reported in Table \ref{tab:combined_cnpe_deepreview} and \ref{tab:combined_pairreview_cyclereviewer} show that graph-based fusion consistently integrates heterogeneous review signals and outperforms each individual method on most metrics. For the combination of CNPE-7B and DeepReview-14B, the fused model achieves the best performance on all six metrics, yielding an average relative improvement of approximately 4.3\% over the strongest standalone baseline, CNPE-7B. For the combination of PairReview and CycleReviewer-7B, the fused model also performs best on most metrics and improves the average performance by approximately 4.9\% relative to the strongest single model, CycleReviewer-7B. By integrating diverse approaches within a graph evidence structure, this paradigm captures substantial complementary information and further improves overall performance. 

Overall, these results suggest that graph-based LLM paper review could be viewed not only as a specific method, but also as a general inference-time evidence-expansion paradigm for paper evaluation that unifies the two dominant formulations of pointwise and pairwise estimation with strong scalability and compatibility.

\section{Methodological Design Details}

\subsection{Beyond Classic GNNs}

Classic GNNs \cite{kipf2016semi,velivckovic2017graph,gilmer2017neural,hamilton2017inductive} treat text embeddings as node features and then learn node representations through iterative neighborhood aggregation. However, our experiments show that such models perform poorly on paper reviewing and fall substantially behind GraphReview. The two paradigms both use graph structure, but GraphReview is not a representation-learning GNN: it elicits task-specific pointwise and pairwise evidence from original paper texts and aggregates the resulting fixed preference graph numerically. This distinction motivates the following analysis of why conventional GNN operators are unsuitable for the task.

We answer this question through a mechanistic analysis. Using GraphSAGE as a representative example, we argue that this family of methods incurs substantial information loss throughout the pipeline and largely undermines the structural clarity and interpretability of the graph. The central issue is that its generic embedding transmission, neighborhood pooling, and latent state mixing operators are not designed for evaluation tasks that require fine-grained semantic understanding and strong alignment with the downstream objective. 

\paragraph{Generic Embedding Transmission}
When edges do not carry explicit semantic features or types, message passing is insensitive to the task-specific relations. Prior methods propagate information mainly according to graph topology. For example, GraphSAGE transmits embeddings with the same generic rule regardless of the specific relation between two papers. In practice, however, paper-paper relations are grounded in meaningful scholarly contexts. Papers from different periods may reflect inheritance or intellectual influence, while contemporaneous papers may represent direct competition. Simply passing information between neighboring nodes cannot capture these distinctions, because the propagation process discards the task-relevant context that gives each relation its meaning. 

\paragraph{Neighborhood Pooling}
Existing aggregation functions compress neighborhood information through permutation-invariant operations, but this compression distorts the actual contribution of neighbors. GraphSAGE uses mean, sum, or max pooling to aggregate incoming messages. As a result, semantically unrelated or even contradictory signals are merged into a single representation, and differences in neighbor importance are flattened by averaging or pooling. This can dilute neighbor-specific evidence, especially under heterophily and suppress higher-order semantic cues, such as argumentative stance and strength of evidence. Such behavior is fundamentally misaligned with our goal of identifying truly contributive papers.

\paragraph{Latent State Mixing}
The update function further mixes intrinsic textual information with semantically heterogeneous external signals, even though these two sources represent different forms of evidence. In paper reviewing, node information reflects the intrinsic quality of the manuscript, whereas edge information reflects the additional value introduced by its domain-level relations to other papers. Traditional GraphSAGE-style methods concatenate these signals into a single embedding and process them jointly in a black-box fashion. This design not only fails to explicitly preserve relation-specific evidence, which is crucial for reviewing, but also weakens attribution, which is critical in the reviewing process. \

Classical GNNs can be effective for relatively shallow tasks such as node classification and link prediction, but they are ill-suited for paper reviewing. This task requires deep semantic understanding, rigorous reasoning, and careful assessment of novelty. Under such demands, conventional GNNs exhibit clear and fundamental limitations.

By addressing the limitations of traditional GNN architectures for paper reviewing, GraphReview establishes a graph-structured evidence framework tailored to this task. Instead of recursively transmitting latent text embeddings, it uses LLMs to elicit explicit comparative evidence on S2FM-selected pairs. It preserves neighbor-specific textual rationales for consolidation and separately converts pairwise outcomes into a directed preference graph. PPR then combines this graph with pointwise priors to obtain globally coupled and attributable predictions. These design choices collectively allow GraphReview to substantially outperform conventional GNN baselines.

\subsection{Beyond LLM Agents}

A natural question is: \textit{why not build a conventional agent-based reviewer, but instead rely on a structured external graph?} Answering this question also clarifies the main challenge addressed in this paper.

At the core of this issue is the fact that paper review is not a closed task, but a process whose judgments become clearer through interaction with an open scientific environment. Classical LLM agent architectures \cite{lu2024ai, weng2025cycleresearcher} typically treat paper review as a task to be solved through decomposition, iterative reflection, or retrieval-augmented analysis. Although retrieved evidence may provide useful background knowledge, it is usually incorporated only as auxiliary context \cite{zhu2025deepreview} rather than as a structured signal that directly shapes evaluation. As a result, these methods often fail to capture the factors that truly determine review quality, such as methodological novelty, the relationship between a paper and prior work, and the temporal lineage of ideas across the literature.

This limitation reflects a deeper mismatch between conventional agent architectures and the nature of the review task. General-purpose agents are designed for broad language tasks through iterative textual reasoning. In contrast, paper review requires reasoning over the external scientific context of a paper, including citation dependencies, conceptual genealogies, and cross-paper comparisons. Treating review as an isolated text understanding problem therefore ignores the critical role of these external signals.

Our method is based on a different perspective. We explicitly model the scientific context as a graph of review evidence and integrate local judgments through sparse pairwise comparison and PPR-based ranking aggregation. This design enables the system to interact with structured knowledge beyond the target paper and to combine review-relevant evidence in a task-aligned manner. We argue that it is precisely this capability, rather than more elaborate in-text reasoning alone, that allows our method to achieve more reliable and accurate paper evaluation.

\subsection{Training Objectives}

The training objective for the LLMs should ideally satisfy two requirements. First, it must capture review signals at both levels: the score distribution over nodes and the binary preference relation over edges. Second, it should retain the semantic evidence associated with expert judgment, namely the evidence supporting each judgment. Direct regression on review scores is insufficient for both purposes. It does not align well with the target distribution and may fail to retain the semantic evidence contained in expert reviews. To address this limitation, we adopt a two-stage training framework.

Paper evaluation is better characterized as a preference comparison problem rather than a long-horizon reasoning task such as mathematical deduction. In peer review, judgments are often formed by directly comparing how two papers present the same aspect, such as clarity, novelty, or empirical support, at corresponding textual locations. This process relies more on evidence localization and comparative attention over text than on repeated Markov state transitions or search over a reasoning tree. Therefore, RLVR methods such as GRPO are not a natural fit for this setting.
This mismatch is both methodological and computational. RLVR-based approaches are typically expensive to train, and their reward signals are often sparse, which may lead to unstable optimization. Moreover, even with reward supervision, large language models can produce reasoning traces that appear plausible without being faithful. Sparse rewards may further exacerbate this issue. By contrast, our RIML training strategy is more naturally aligned with preference optimization methods, whose objective is to match human comparative judgments rather than solve a multi-step reasoning problem.

Once the task is framed as preference alignment, the central question becomes what should be aligned. Supervising only the final decision may be insufficient, because the relevant signal does not reside solely in the conclusion. If training targets only the final outcome, the model may learn to imitate answer patterns without internalizing the underlying criteria. Therefore, we distill the global explanatory rationale behind a preference judgment into the earliest stage of generation. We first generate the explanation supporting the preference from the teacher model, and then use this distilled signal to fine-tune an open-source student model. After training, the model's initial token can implicitly encode a compact representation of the subsequent explanatory reason, helping the model reach preference decisions without relying on lengthy explicit reasoning at inference time. Our ablation results show that directly applying preference optimization without this explanatory cold-start SFT causes the model to overlook critical evidence and yields worse performance.

Additionally, our approach is primarily intended for review scenarios where supervision is derived from human rating signals. Under the assumption that a paper's derivations and experiments are broadly sound and follow standard academic practice, the model is particularly well suited to capturing human scientific taste and assessing academic value. Its current formulation focuses less on settings that require rigorous verification of formal derivations or detailed validation of experimental results, which we leave for future work.

\section{Case Study}
Table \ref{tab:case_study} presents a complete case study using a representative paper. It covers the full workflow from the original manuscript to the final review report, offering an end-to-end walkthrough of our evaluation pipeline.

\section{Prompts}
We release all prompt templates used in our system in Tables \ref{tab:criteria_optimization_prompt}--\ref{tab:text_consolidation_prompt} to document instruction design across evaluation stages and support reproducibility and analysis.

\begin{table*}[t]
    \renewcommand{\arraystretch}{0.82}
	\setlength{\belowcaptionskip}{-15pt} 
	\setlength{\abovecaptionskip}{0pt} 
	\centering
	\begin{mdframed}[
		linewidth=0.8pt,
		roundcorner=6pt,
		backgroundcolor=white,
		linecolor=gray!60,
		innerleftmargin=10pt,
		innerrightmargin=10pt,
		innertopmargin=5pt,
		innerbottommargin=5pt
		]
		\small
		\textbf{Case Study} \\ \\
	We take paper 68J0pJFCi3 as an example and illustrate the full process it undergoes throughout evaluation.
		
		\begin{mdframed}[
			linewidth=0.8pt,
			roundcorner=6pt,
			backgroundcolor=orange!3,
			linecolor=gray!60,
			innerleftmargin=10pt,
			innerrightmargin=10pt,
			innertopmargin=5pt,
			innerbottommargin=5pt
			]
			\small
			\textbf{Original Text}
			\begin{Verbatim}[breaklines=true, breakanywhere=true]
# ON REPRESENTING CONVEX QUADRATICALLY CON-STRAINED QUADRATIC PROGRAMS VIA GRAPH NEURAL NETWORKS

Anonymous authors Paper under double-blind review

# ABSTRACT

Convex quadratically constrained quadratic programs (QCQPs) involve finding a solution ...

# 1 INTRODUCTION

Quadratic programs (QPs) are a ...
			\end{Verbatim}
		\end{mdframed}
		
		On the one hand, we directly score this central paper to obtain an initial estimate of its quality at the node level.
		
		\begin{mdframed}[
			linewidth=0.8pt,
			roundcorner=6pt,
			backgroundcolor=orange!3,
			linecolor=gray!60,
			innerleftmargin=10pt,
			innerrightmargin=10pt,
			innertopmargin=5pt,
			innerbottommargin=5pt
			]
			\small
			\textbf{Focal Paper Scoring $\times 1$} \\ \\
			\textit{Paper (68J0pJFCi3, 2025)}
			\begin{Verbatim}[breaklines=true, breakanywhere=true]
<|SCORING_TOKEN|> This paper, \"On Representing Convex Quadratically Constrained Quadratic Programs via Graph Neural Networks,\" proposes a novel tripartite graph representation for convex QCQPs and establishes theoretical foundations regarding the expressiveness of message passing GNNs applied to this representation ...
			\end{Verbatim}
		\end{mdframed}
		
		On the other hand, we iteratively introduce additional edges and compare the paper with other papers.
		
		\begin{mdframed}[
			linewidth=0.8pt,
			roundcorner=6pt,
			backgroundcolor=orange!3,
			linecolor=gray!60,
			innerleftmargin=10pt,
			innerrightmargin=10pt,
			innertopmargin=5pt,
			innerbottommargin=5pt
			]
			\small
			\textbf{Synchronic and Diachronic Pairs Comparison $\times 2T$} \\ \\
			\textit{Paper A (68J0pJFCi3, 2025), Paper B (ULleq1Dtaw, 2024)}
			\begin{Verbatim}[breaklines=true, breakanywhere=true]
<|COMPARISON_TOKEN|> The choice of Paper A as the higher quality paper is based on a comprehensive analysis across six key criteria. While both papers present novel contributions to video-language modeling, Paper A demonstrates superior ...
			\end{Verbatim}
			\textit{Paper A (XnVttczoAV, 2025), Paper B (68J0pJFCi3, 2025)}
			\begin{Verbatim}[breaklines=true, breakanywhere=true]
 ...
			\end{Verbatim}
		\end{mdframed}
		
		We convert these special tokens into numerical values and integrate the corresponding signals via the PPR algorithm, producing a final numerical score. Meanwhile, the textual content is processed in parallel through a consolidation step.
		
		\begin{mdframed}[
			linewidth=0.8pt,
			roundcorner=6pt,
			backgroundcolor=orange!3,
			linecolor=gray!60,
			innerleftmargin=10pt,
			innerrightmargin=10pt,
			innertopmargin=5pt,
			innerbottommargin=5pt
			]
			\small
			\textbf{Final Report}
			\begin{Verbatim}[breaklines=true, breakanywhere=true]
**Ranking:** 225/500  
**Decision:** Reject  

**Summary**:  
This paper introduces a tripartite graph representation for convex Quadratically Constrained Quadratic Programs (QCQPs) and presents a theoretical framework linking this representation to the expressiveness of message passing Graph Neural Networks (GNNs). While the paper ...

**Advantages**:  
1.  **Relevant Problem Selection:** The task of representing and solving convex QCQPs ...

**Disadvantages**:  
1.  **Incremental and Poorly Justified Technical Contribution:** The proposed ...

**Questions**:
1.  The theorem states that a GNN can universally approximate the ...

**Suggestions**:
1.  **Complete and Justify the Theoretical Claims:** The authors ...
			\end{Verbatim}
		\end{mdframed}
		
		Finally, we produce a complete evaluation report with an associated score for paper 68J0pJFCi3.
		
	\end{mdframed}
	\caption{Case study. An example illustrating the complete workflow for evaluating a paper.}
	\label{tab:case_study}
\end{table*}

\begin{table*}[t]
	\centering
    \renewcommand{\arraystretch}{0.82}
	\setlength{\belowcaptionskip}{-15pt} 
	\setlength{\abovecaptionskip}{0pt} 
	\begin{mdframed}[
		linewidth=0.8pt,
		roundcorner=6pt,
		backgroundcolor=gray!6,
		linecolor=gray!60,
		innerleftmargin=10pt,
		innerrightmargin=10pt,
		innertopmargin=5pt,
		innerbottommargin=5pt
	]
		\small
		\textbf{Criteria Optimization Prompt}
		\begin{Verbatim}[breaklines=true, breakanywhere=true]
You are an expert prompt optimizer.
Your task is to optimize the {criteria} so that the language model can generate better responses.
Do not provide any information related to the output format or output requirements; analyze only the content.
The <Criteria> you provide must be structured (Use 1. 2. 3. ...), expressed clearly and accurately.
You must only return the optimized <Criteria>.
Prompt:
```
{prompt}
```
Current <Criteria> (Empty if none):
```
{criteria}
```
		\end{Verbatim}
	\end{mdframed}
	\caption{Criteria optimization prompt. Used to create a new prompt based on an existing prompt.}
	\label{tab:criteria_optimization_prompt}
\end{table*}

\begin{table*}[t]
	\centering
    \renewcommand{\arraystretch}{0.82}
	\setlength{\belowcaptionskip}{-15pt} 
	\setlength{\abovecaptionskip}{0pt} 
	\begin{mdframed}[
		linewidth=0.8pt,
		roundcorner=6pt,
		backgroundcolor=gray!6,
		linecolor=gray!60,
		innerleftmargin=10pt,
		innerrightmargin=10pt,
		innertopmargin=5pt,
		innerbottommargin=5pt
		]
		\small
		\textbf{Answer Evaluation Prompt}
		\begin{Verbatim}[breaklines=true, breakanywhere=true]
You are an expert answer evaluator.
Your task is to compare which of the two answers is of higher quality.
You must only return the character (A/B) representing the quality.
Answer A:
```
{answer_A}
```
Answer B:
```
{answer_B}
```
Better: 
		\end{Verbatim}
	\end{mdframed}
	\caption{Answer evaluation prompt. Used to evaluate the quality of answers generated by evolved criteria.}
	\label{tab:answer_evaluation_prompt}
\end{table*}

\begin{table*}[t]
	\centering
    \renewcommand{\arraystretch}{0.82}
	\setlength{\belowcaptionskip}{-25pt} 
	\setlength{\abovecaptionskip}{0pt} 
	\begin{mdframed}[
		linewidth=0.8pt,
		roundcorner=6pt,
		backgroundcolor=gray!6,
		linecolor=gray!60,
		innerleftmargin=10pt,
		innerrightmargin=10pt,
		innertopmargin=5pt,
		innerbottommargin=5pt
		]
		\small
		\textbf{Data Construction Prompt (Scoring)}
		\begin{Verbatim}[breaklines=true, breakanywhere=true]
You are an expert reviewer. 
Evaluating a paper by its quality.
Your analysis must follow the following criteria:
{criteria}
Your answer should be about 2000 words.
Do not use bold or any other symbols.
Your answer must always begin with a scored number, and there must be no other text before it.
Do not mention any issues regarding the truncation of submitted content; truncation does not constitute part of the analysis.
The correct score for this paper is {ground_truth}.
You must give this score and provide a convincing explanation for why this score is appropriate. 
		\end{Verbatim}
		\begin{Verbatim}[breaklines=true, breakanywhere=true]
Evaluate the following paper and provide a score using the following scale:
0: strong reject
1: reject, not good enough
2: marginally below the acceptance threshold
3: marginally above the acceptance threshold
4: accept, good paper
5: strong accept, should be highlighted at the conference
Here is the paper:
```
{paper_text}
```
Please provide your score first as a single number (0-5), then explain your reasoning.
Score:
		\end{Verbatim}
	\end{mdframed}
	\caption{Data construction prompt for paper scoring.}
	\label{tab:data_construction_prompt_scoring}
\end{table*}

\begin{table*}[t]
	\centering
    \renewcommand{\arraystretch}{0.82}
	\setlength{\belowcaptionskip}{-15pt} 
	\setlength{\abovecaptionskip}{0pt} 
	\begin{mdframed}[
		linewidth=0.8pt,
		roundcorner=6pt,
		backgroundcolor=gray!6,
		linecolor=gray!60,
		innerleftmargin=10pt,
		innerrightmargin=10pt,
		innertopmargin=5pt,
		innerbottommargin=5pt
		]
		\small
		\textbf{Data Construction Prompt (Comparison)}
		\begin{Verbatim}[breaklines=true, breakanywhere=true]
You are an expert reviewer. 
Analyze the following two papers separately, then indicate which one is of higher quality.
Your analysis must follow the following criteria:
{criteria}
Your answer should be about 1000 words.
Do not use bold or any other symbols.
Your answer must always begin with a choice (A or B), and there must be no other text before it.
Do not mention any issues regarding the truncation of submitted content; truncation does not constitute part of the analysis.
The correct choice is {ground_truth}.
You must output this choice and provide a convincing explanation for why this choice is appropriate.
		\end{Verbatim}
		\begin{Verbatim}[breaklines=true, breakanywhere=true]
Compare the following two papers and decide which one is better in quality.
Paper A:
```
{paper_text_a}
```
Paper B:
```
{paper_text_b}
```
Please provide your choice first as a single letter (A or B), then explain your reasoning.
Choice: 
		\end{Verbatim}
	\end{mdframed}
	\caption{Data construction prompt for pair comparison.}
	\label{tab:data_construction_prompt_comparion}
\end{table*}

\begin{table*}[t]
	\centering
    \renewcommand{\arraystretch}{0.82}
	\setlength{\belowcaptionskip}{-15pt} 
	\setlength{\abovecaptionskip}{0pt} 
	\begin{mdframed}[
		linewidth=0.8pt,
		roundcorner=6pt,
		backgroundcolor=gray!6,
		linecolor=gray!60,
		innerleftmargin=10pt,
		innerrightmargin=10pt,
		innertopmargin=5pt,
		innerbottommargin=5pt
		]
		\small
		\textbf{Text Consolidation Prompt}
		\begin{Verbatim}[breaklines=true, breakanywhere=true]
You are an expert academic reviewer and research analyst.
Your task is to produce an enhanced review of a single paper by integrating valuable comparative insights.
Instructions:
1. Use the provided `single_paper_review` as the primary foundation and preserve its core judgment unless the comparative evidence clearly justifies adjustment.
2. For each entry in `related_pairs`, briefly extract only the most relevant information from `pair_comparison`, especially comparative strengths, weaknesses, missing validations, or clearer methodological standards that are directly useful for evaluating this paper.
3. Integrate these insights naturally into the `single_paper_review`, citing the relevant literature in the merged text. Citation format: e.g. `(#0, 2025)`. Use comparisons selectively and only when they strengthen or clarify the review.
4. You must output content related to `ranking` and `decision` at first, e.g. `**Ranking:** (0/500)` and `**Decision:** Accept`. Make sure the ranking, decision, and all arguments are fully consistent with each other after revision.
5. Structure the review clearly into layered sections: first give an overall assessment, then list the most important strengths, then the most important weaknesses, and finally concrete questions/suggestions. Avoid repetition across sections.
6. The questions and suggestions proposed must all be highly practical, specific, feasible, and directly actionable for the authors to address.
7. Keep the tone professional, evidence-based, and concise. Avoid exaggerated claims or unsupported criticism.
8. Only output the merged text. Do not include any other content.
Here is all the content related to the paper:
```
{json_str}
```
The output format you need to follow:
```
**Ranking:**
**Decision:**
**Summary**:
**Advantages**:
**Disadvantages**:
**Questions**:
**Suggestions**:
```
		\end{Verbatim}
	\end{mdframed}
	\caption{Text consolidation prompt. Merge the texts to generate complete review.}
	\label{tab:text_consolidation_prompt}
\end{table*}

\begin{table*}[t]
	\centering
    \renewcommand{\arraystretch}{0.82}
	\setlength{\belowcaptionskip}{-15pt} 
	\setlength{\abovecaptionskip}{0pt} 
	\begin{mdframed}[
		linewidth=0.8pt,
		roundcorner=6pt,
		backgroundcolor=gray!6,
		linecolor=gray!60,
		innerleftmargin=10pt,
		innerrightmargin=10pt,
		innertopmargin=5pt,
		innerbottommargin=5pt
		]
		\small
		\textbf{Text Evaluation Prompt}
		\begin{Verbatim}[breaklines=true, breakanywhere=true]
You are a senior meta-reviewer evaluating the quality of peer-review reports for scientific conferences and journals.
Your task is to compare two reviews of the same paper and decide which review is stronger as a scientific review.
RESPONSE FORMAT REQUIREMENTS:
1. Return a valid JSON object only, with no extra text.
2. The JSON must contain exactly these ten keys:
- "technical_depth"
- "technical_depth_reason"
- "evidence_grounding"
- "evidence_grounding_reason"
- "scientific_rigor"
- "scientific_rigor_reason"
- "revision_utility"
- "revision_utility_reason"
- "overall_preference"
- "overall_preference_reason"
3. For each label key, the value must be exactly one of: "A", "B", or "Tie".
4. For each reason key, the value must be one brief sentence of at most 22 words.
5. Do not output anything except the JSON object.
EVALUATION DIMENSIONS:
1. technical_depth: Which review engages more deeply with the paper's technical substance, such as method details, assumptions, derivations, proofs, experiments, evaluation design, complexity, or implementation?
2. evidence_grounding: Which review ties its judgments more directly to paper-specific evidence, claims, equations, tables, figures, baselines, metrics, or clearly missing analyses?
3. scientific_rigor: Which review more rigorously evaluates validity, claim-evidence alignment, fairness of comparisons, reproducibility, completeness of argumentation, and whether the paper's conclusions are actually supported?
4. revision_utility: Which review gives more useful and actionable guidance for improving the paper, especially through concrete, acceptance-relevant revisions?
5. overall_preference: Overall, which review is more valuable for editorial decision-making and author revision, considering technical insight, evidence-based criticism, exposure of substantive weaknesses, and usefulness for improving the paper?
CORE JUDGING PRINCIPLES:
1. Judge only the quality of the reviews, not the quality of the paper.
2. Prefer reviews that identify central technical weaknesses, unsupported claims, weak evidence, missing controls, incomplete proofs, confounds, unfair baselines, or reproducibility gaps.
3. Prefer paper-specific critique over generic balance, polished wording, soft tone, or formulaic reviewing language.
4. Do not reward a review merely for sounding more diplomatic, more moderate, more balanced, or more polished.
5. Do not penalize a review merely for being critical, forceful, technically dense, or highly detailed, if its concerns are concrete and grounded in the paper.
6. Strong reviews often directly explain why the current evidence is insufficient for the paper's claims.
7. Comparative references to related work may be useful when they concretely support criticism about novelty, baselines, theory, or evaluation standards; do not dismiss them automatically unless they substantially replace paper-specific analysis.
8. Ignore superficial differences in politeness or rhetorical style unless they materially affect scientific clarity or introduce unsupported claims.
9. If one review is sharper but better exposes acceptance-relevant weaknesses, it can be better overall even if it is less smooth stylistically.
10. In close cases, prefer the review that better identifies substantive risks to validity or acceptance.
11. If the two reviews are difficult to distinguish in quality, choose "Tie" rather than defaulting to "A" due to positional bias.
12. If one review is empty, select the other review accordingly.
		\end{Verbatim}
	\color{red}\noindent\hfill\textit{Continued on next page.}
	\end{mdframed}
	\caption{Text evaluation prompt. Comparing the text quality with other approaches.}
	\label{tab:text_evaluation_prompt}
\end{table*}

\begin{table*}[t] \ContinuedFloat
	\centering
    \renewcommand{\arraystretch}{0.82}
	\setlength{\belowcaptionskip}{-15pt} 
	\setlength{\abovecaptionskip}{0pt} 
	\begin{mdframed}[
		linewidth=0.8pt,
		roundcorner=6pt,
		backgroundcolor=gray!6,
		linecolor=gray!60,
		innerleftmargin=10pt,
		innerrightmargin=10pt,
		innertopmargin=5pt,
		innerbottommargin=5pt
		]
		\small
		\textbf{Text Evaluation Prompt} \textit{(Continued)}
		\begin{Verbatim}[breaklines=true, breakanywhere=true]
Compare Review A and Review B as peer-review reports for the same paper.
Return only a valid JSON object with exactly these keys:
"technical_depth"
"technical_depth_reason"
"evidence_grounding"
"evidence_grounding_reason"
"scientific_rigor"
"scientific_rigor_reason"
"revision_utility"
"revision_utility_reason"
"overall_preference"
"overall_preference_reason"
For each label key, output exactly one of: "A", "B", or "Tie".
For each reason key, output one brief sentence of at most 22 words.
Important instructions:
1. Judge only the quality of the reviews, not the paper itself.
2. Prefer reviews that identify important technical flaws, unsupported claims, weak evidence, missing experiments, incomplete proofs, unfair comparisons, or reproducibility issues.
3. Prefer paper-specific, evidence-linked criticism over smoother wording or more diplomatically balanced tone.
4. Do not reward a review merely for sounding more polished, more measured, or more conventionally editorial.
5. A sharper or more critical review can be better if its concerns are concrete, technically meaningful, and grounded in the paper.
6. Related-work comparisons may be useful when they concretely support criticism about novelty, baselines, theory, or evaluation standards.
7. In close cases, overall_preference should favor the review that better exposes acceptance-relevant weaknesses and better helps an editor decide.
8. If the two reviews are difficult to distinguish in quality, output "Tie" rather than defaulting to "A" because of positional bias.
Review A:
{review_a}
Review B:
{review_b}
		\end{Verbatim}
	\end{mdframed}
	\caption{Text evaluation prompt (Continued). Comparing the text quality with other approaches.}
	\label{tab:text_evaluation_prompt_continue}
\end{table*}

\end{document}